\documentclass{article}

\usepackage[final]{neurips_2024}

\usepackage[utf8]{inputenc} 
\usepackage[T1]{fontenc}    
\usepackage{hyperref}       
\usepackage{url}            
\usepackage{booktabs}       
\usepackage{amsfonts}       
\usepackage{nicefrac}       
\usepackage{microtype}      
\usepackage{xcolor}         

\usepackage{algorithmic}
\usepackage{graphicx}
\usepackage{textcomp}

\usepackage{multirow}
\usepackage{pifont}
\usepackage{bm}
\usepackage{amsmath} 
\usepackage{amssymb}
\usepackage{amsthm}

\usepackage{bbm}
\usepackage{graphicx}

\usepackage{caption}
\usepackage{float} 
\usepackage{subcaption}
\usepackage{algorithm}
\usepackage{algorithmic}

\usepackage[capitalize,noabbrev,nameinlink]{cleveref} 
\crefname{lem}{Lemma}{Lemmas}
\crefname{defin}{Definition}{Definitions}
\crefname{thm}{Theorem}{Theorems}

\newtheorem{theorem}{\bf Theorem}
\newtheorem{lemma}{\bf Lemma}

\newtheorem{definition}{\bf Definition}

\definecolor{lightgray}{gray}{0.9}

\title{L{\large AMBDA}: Learning Matchable Prior For Entity Alignment with Unlabeled Dangling Cases}

\author{%
	Hang Yin, Liyao Xiang\thanks{Corresponding author: Liyao Xiang, xiangliyao08@sjtu.edu.cn.}\\
	Shanghai Jiao Tong University\\
	Shanghai, China \\
	\texttt{\{yinhang\_SJTU, xiangliyao08\}@sjtu.edu.cn} \\
	\And
	Dong Ding\\
	Shanghai Jiao Tong University\\
	Shanghai, China \\
	\texttt{18916162516@sjtu.edu.cn}
	\And
	Yuheng He, Yihan Wu, Pengzhi Chu, Xinbing Wang\\
	Shanghai Jiao Tong University\\
	Shanghai, China \\
	\texttt{\{heyuheng, caracalla, pzchu, xwang8\}@sjtu.edu.cn}
	\And
	Chenghu Zhou\\ 
	Chinese Academy of Sciences\\
	Beijing, China \\
	\texttt{zhouchsjtu@gmail.com}
}

\begin{document}
	
	\maketitle
	
\begin{abstract}
	We investigate the entity alignment (EA) problem with unlabeled dangling cases, meaning that partial entities have no counterparts in the other knowledge graph (KG), yet these entities are unlabeled. The problem arises when the source and target graphs are of different scales, and it is much cheaper to label the matchable pairs than the dangling entities. To address this challenge, we propose the framework \textit{Lambda} for dangling detection and entity alignment. Lambda features a GNN-based encoder called KEESA with a spectral contrastive learning loss for EA and a positive-unlabeled learning algorithm called iPULE for dangling detection. Our dangling detection module offers theoretical guarantees of unbiasedness, uniform deviation bounds, and convergence. Experimental results demonstrate that each component contributes to overall performances that are superior to baselines, even when baselines additionally exploit 30\% of dangling entities labeled for training.
\end{abstract}

\section{Introduction}\label{Introduction}
Entity alignment is a problem that seeks entities referring to the same real-world identity across different knowledge graphs (KGs), and is widely deployed in fields such as knowledge fusion, question-answering, web mining, etc. To address the issue, embedding-based methods have been proposed to capture entity similarity in the embedding space through translation-based \cite{luo2022accurate, xu2019cross, liu2022selfkg} or graph neural network (GNN)-based \cite{wang2018cross, sun2020knowledge, mao2020relational, mao2020mraea, sun2018bootstrapping} models. Particularly, if the entities do not have counterparts on another KG, the entities are referred to as \textit{dangling entities}, as shown in Fig.~\ref{fig:dangling-alignment}. 

In many real-world scenarios, the labels for the dangling entities on KGs are often missing, as those labels are much harder to acquire. For example, in KG plagiarism detection, it is relatively easy to align entity pairs that both exist in KGs, but one would have to traverse all possible pairs to conclude an entity is not paired. Hence \textit{EA with unlabeled dangling entities} is a hard but realistic problem. The problem even worsens in EA on KGs of different scales where the dangling entities take a large proportion of all nodes.  

Despite that many works have been investigating the EA problem with dangling entities, few have focused on EA with unlabeled dangling cases. We list closely related works in Table~\ref{table:summary}. The work of \cite{sun2021knowing} extends the conventional EA methods MTransE \cite{chen2017multilingual} and AliNet \cite{sun2020knowledge} to the case with dangling entities, and thus require a portion of labeled dangling entities for training. Based on their works, MHP  \cite{liu2022dangling} has improved performance with additional knowledge, i.e., the high-order proximities information for alignment. Both UED \cite{luo2022accurate} and SoTead \cite{luo2022semi} are unsupervised schemes that rely on side information such as entity name/attribute as global alignment information. Different from prior works, we consider a stricter case where neither side information nor any labeled dangling entities are known, as side information often leads to name-bias \cite{liu2023unsupervised, zhang2020industry} while labels for dangling entities are hard to obtain.

\begin{figure}[htbp]
	\begin{minipage}{0.49\linewidth}
		\centering
		\includegraphics[height = 0.128\textheight, width = 0.7\textwidth]{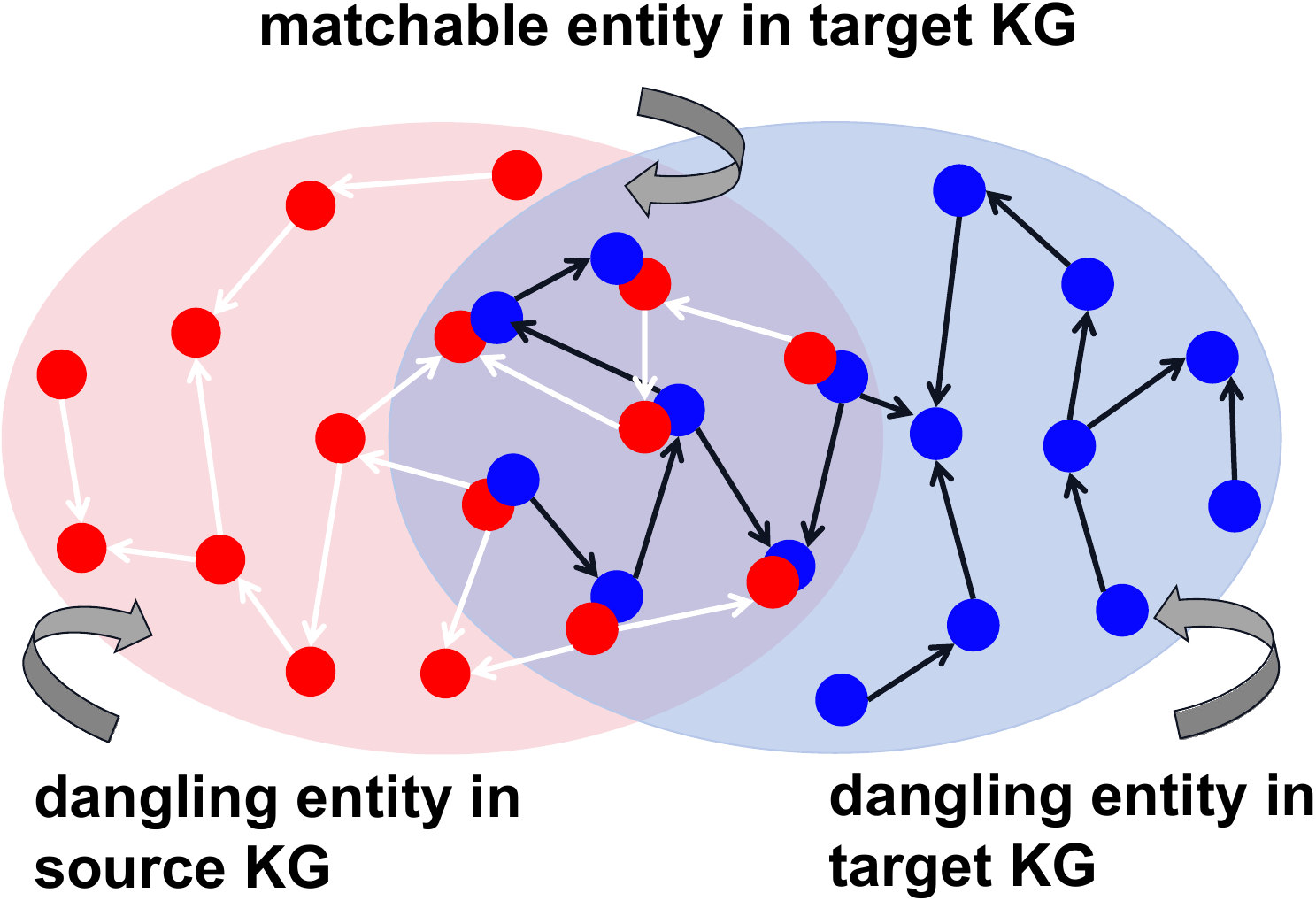}
		\caption{\centering Examples of dangling entities.}
		\label{fig:dangling-alignment}
	\end{minipage}
	\begin{minipage}{0.5\linewidth}
		\centering
		\renewcommand\arraystretch{1}
		\setlength{\tabcolsep}{3pt}
		\begin{tabular}{@{}lcc@{}}
			\toprule
			\multicolumn{1}{l}{Method} & \multicolumn{1}{c}{Side Info} & \multicolumn{1}{c}{Dangling Labels} \\ 
			\hline
			\cite{sun2021knowing} w/ MTransE &        \ding{55}               &        \ding{52}               \\
			\cite{sun2021knowing} w/ AliNet &         \ding{55}              &         \ding{52}              \\
			UED \cite{luo2022accurate}  &            \ding{52}           &           \ding{55}            \\
			SoTead \cite{luo2022semi}   &         \ding{52}              &         \ding{55}             \\
			MHP \cite{liu2022dangling}       &            \ding{55}             &      \ding{52} + high-order info    \\  
			Our Work   &       \ding{55}                 &      \ding{55}    \\
			\bottomrule                 
		\end{tabular}
		\captionof{table}{Different EA models with dangling cases.}\label{table:summary}
		\label{Tab: mic}
	\end{minipage}  
\end{figure}

The EA with unlabeled dangling entities faces unique challenges: \textit{first,} the unlabeled dangling entities would cause erroneous information to propagate through neighborhood aggregation if applying conventional GNN-based embedding methods, negatively affecting the dangling detection and alignment of matchable entities. \textit{Second,} the absence of labeled dangling entities makes its feature distribution non-observable, requiring the model to distinguish potential dangling entities while learning the representation of matchable entities. Hence the EA problem has to be solved with mere positive (matchable entities with labels) and unlabeled samples, yet without any prior knowledge of the distribution of the nodes.

We tackle the first challenge by proposing a novel GNN-based EA framework. To eliminate the `pollution' of dangling entities, the adaptive dangling indicator has been applied globally for selective aggregation. Relation projection attention is designed to combine both entity and relation information for a more comprehensive representation. The designed spectral contrastive learning loss disentangles the matchable entities from dangling ones while portraying a unified embedding space for entity alignment.

As to the second challenge, we first derive an unbiased risk estimator and a tighter uniform deviation bound for the positive-unlabeled (PU) learning loss. However, such an estimator still requires prior knowledge of the proportion of positive entities among all nodes. Thus we develop an iterative strategy to estimate such prior knowledge while training the classifier with a PU learning loss. The prior estimation could also be seen as dangling entity detection; if too few entities are determined to be matchable, one can stop all subsequent procedures and decide the two KGs cannot be aligned.

Our framework Lambda is provided in Fig.~\ref{fig:framework} where there are two phases corresponding to two trained models --- dangling detection and entity alignment. Both phases share one GNN-based encoder and the spectral contrastive learning loss. The dangling detection additionally uses a positive-unlabeled learning loss. The GNN-based encoder contains both the intra-graph and the cross-graph representation learning modules. After the dangling detection, the estimated proportion of matchable entities is figured for judging whether two KGs could be aligned. Only aligned KGs are sent for EA model training, and then only first-phase predicted matchable entity embeddings are obtained from the EA model for inference. Finally, we select pairs of entities that are mutually nearest by their embeddings as aligned pairs.  

Highlights of our contributions are as follows: we raise the challenging problem of EA with unlabeled dangling entities for the first time. To resolve the issue, we propose the framework Lambda featured by a GNN-based encoder called KEESA with spectral contrastive learning for EA and a positive-unlabeled learning algorithm for dangling detection called iPULE. We provide a theoretical analysis of PU learning on the unbiasedness, uniform deviation bound, and convergence. Experiments on a variety of real-world datasets have demonstrated our alignment performance is superior to baselines, even the baselines with 30\% labeled dangling entities. Our code is available on github\footnote{https://github.com/Handon112358/NeurIPS\_2024\_Learning-Matchable-Prior-For-Entity-Alignment-with-Unlabeled-Dangling-Cases}.

\begin{figure*}[t]
	\centering
	\includegraphics[height = 0.31\textheight, width = 0.95\textwidth]{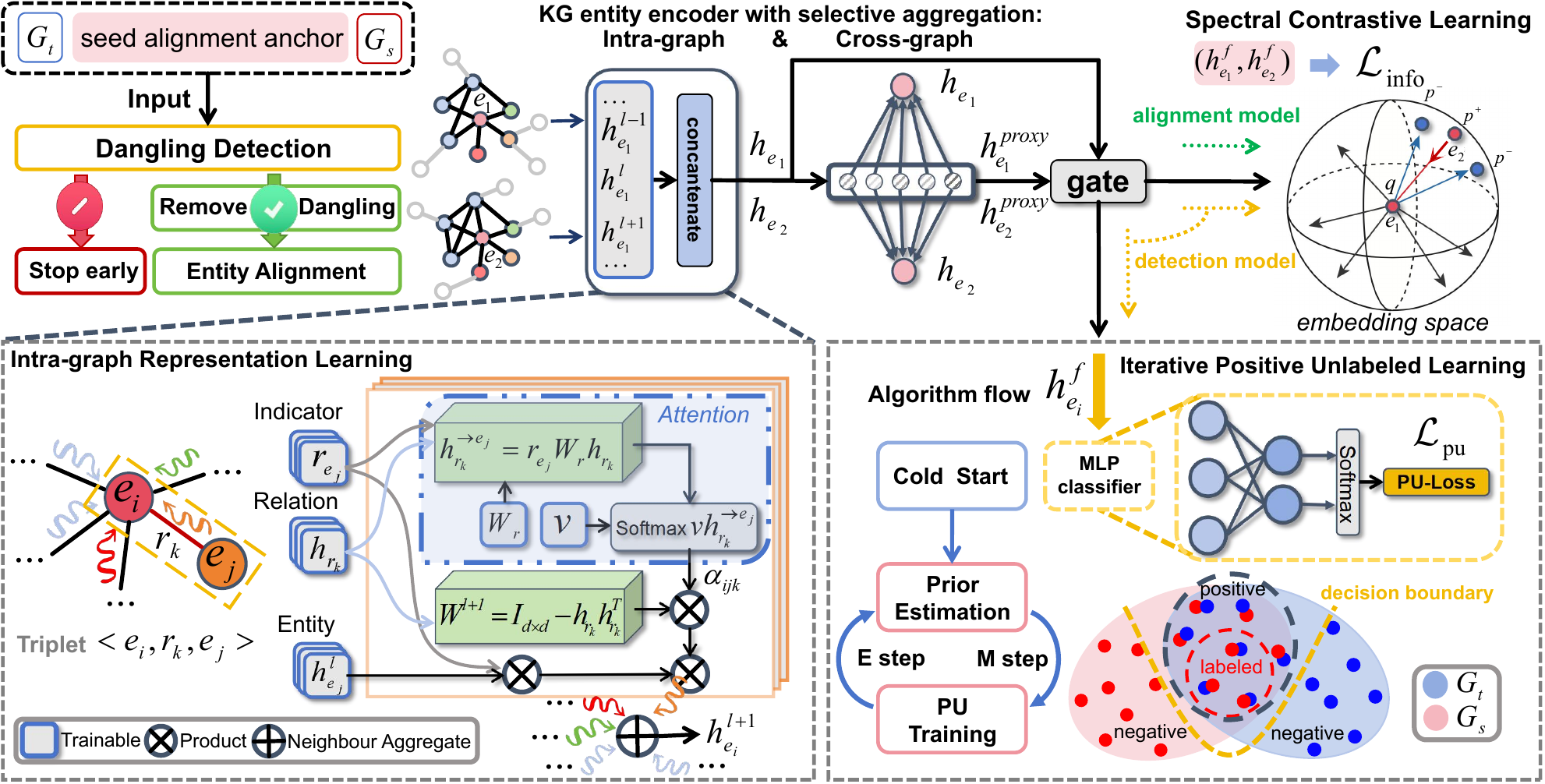}
	\caption{The illustration of our framework.}
	\label{fig:framework}
\end{figure*}

\section{Preliminaries and Related Work}

\subsection{Entity Alignment}
Embedding-based entity alignment methods have evolved rapidly and are gradually becoming the mainstream approach of EA in recent years due to their flexibility and effectiveness \cite{li2023vision}, which aim to encode KGs into low-dimensional embedding space to capture the similarities of entities \cite{li2023active, ji2021survey}. It could be divided into translation-based \cite{luo2022accurate, xu2019cross, liu2022selfkg} and GNN-based \cite{wang2018cross, sun2020knowledge, mao2020relational, mao2020mraea, sun2018bootstrapping}. 

Previous EA methods mostly assume a one-to-one correspondence between two KGs, but such an assumption does not always hold and thus leads to a performance drop in real-world cases \cite{sun2020benchmarking}. Notably, Sun \emph{et al.} \cite{sun2021knowing} as a pioneering work modeled it upon a supervised setting, i.e., a small set of aligned entities and labeled dangling entities. On this basis, MHP \cite{liu2022dangling} employed more dangling information concerning high-order proximities in both training and inference. UED \cite{luo2022accurate} and SoTead \cite{luo2021graph} propose an unsupervised translation-based method for joint entity alignment and dangling entity detection without labeled dangling entities, while the practical problem of matching cost is ignored.

\subsection{Positive-Unlabled Learning}
\label{sec:PU_setting}
Let $X\in\mathbb{R}^d, d\in\mathbb{N}, \mathrm{~and~}Y\in\{\pm1\}$ be the input and output random variables. We also define $p(x,y)$ to be the \textit{joint probability density of $(X, Y)$}, $p_\mathrm{p}(x)=p(x\mid y=+1)$, $p_\mathrm{n}(x)=p(x\mid y=-1)$ to be the \textit{P (Positive) and N (Negative) marginals (a.k.a., class-conditional densities)}, and $p_\mathrm{u}(x)$ be the \textit{U (Unlabeled) marginal}. The \textit{class-prior probability} is expressed as $\pi_{p}=p(y=+1)$, which is assumed to be known throughout the paper, and can be estimated from known datasets \cite{kiryo2017positive}.

The PU learning problem setting is as follows: the positive and unlabeled data are sampled independently from $p_\mathrm{p}(x)$ and $p_\mathrm{u}(x)$ as $\mathcal{X}_\mathrm{p}=\{x_i^\mathrm{p}\}_{i=1}^{n_\mathrm{p}}\sim p_\mathrm{p}(x)$ and $\mathcal{X}_\mathrm{u}=\{x_i^\mathrm{u}\}_{i=1}^{n_\mathrm{u}}\sim p_\mathrm{u}(x)$, and a classifier is trained from $\mathcal{X}_\mathrm{p}$ and $\mathcal{X}_\mathrm{u}$, in contrast to learning a classifier telling negative samples apart from positive ones. The general assumption of the previous work is to let the unlabeled distribution be equal to the overall data distribution, i.e., $p_\mathrm{u}(x) = p(x)$ since $p_\mathrm{u}(x)$ cannot be obtained, but the assumption hardly holds in many real-world scenarios, for example transductive learning, making methods in \cite{kiryo2017positive, niu2016theoretical} infeasible.

\section{Selective Aggregation with Spectral Contrastive Learning}\label{sec:KEESA}
\textbf{Notation:} Source and target KG $G_s=(E_s, R_s, T_s)$, $G_t=(E_t, R_t, T_t)$ stored in triples $<$entity, relation, entity$>$: entities $E$, relations $R$, and triples $T \subseteq E \times R \times E$, $E_s=D_s \cup M_s, E_t=D_t \cup M_t$, where $D$ denotes dangling and $M$ denotes matchable. A set of pre-aligned anchor node pairs are $S=\left\{(u,v)|u \in M_s, v \in M_t, u \equiv v\right\}$. (see appendix~\ref{app_0} for more details). 

We start by introducing the KEESA (\textbf{K}G \textbf{E}ntity \textbf{E}ncoder with \textbf{S}elective \textbf{A}ggregation).
\subsection{KG Entity Encoder with Selective Aggregation}
\textbf{Adaptive Dangling Indicator \& Relation Projection Attention.}
Real-world EA tasks often involve graphs with dangling distortion \cite{surisetty2022reps, belhajjame2023online}. Conventional GNN aggregation will `pollute' matchable entities' embeddings with dangling. However, a hard dangling indicator for the entity is over-confident as only approximate results can be obtained without labels. Incorrect indicators may lead to inappropriate aggregation and thus destruction of the KG structure.  Instead, we apply a learnable scalar weight $r_{e_j}$ for each $e_i$'s neighboring message:
\begin{equation}
	\bm{h}^{l+1}_{e_i} = \sigma\left( \sum_{e_j\in \mathcal{N}_{e_i} \cup \{e_i\}} \underbrace{\tanh(r_{e_j})}_{\text{adaptive dangling indicator}} \alpha_{i,j} W^{l+1} \bm{h}^{l}_{e_j}\right),\label{intra}
\end{equation}
where $\tanh$ serves to normalize $r_{e_j}$ to the range of $[-1, 1].$ The initialization of $r_{e_j}$ is critical, please see the implementation details for more.

As compressed feature of $e_j$ --- $r_{e_j}$ is a plain scalar, we link relation $r_k$'s embedding $\bm{h}_{r_k}$ to the associated entity $e_j$ by $\bm{h}_{r_k}^{\to e_j}$ for capturing more comprehensive attention. A matrix $W_r \in \mathbb{R}^{d \times d}$ with an orthogonal regularizer $L_o$ is applied to $\bm{h}_{r_k}$ to perform projection while preserving its norm for better convergence: 
\begin{equation}
	\bm{h}_{r_k}^{\to e_j} = r_{e_j} W_r \bm{h}_{r_k} \text{\quad and \quad} 
	L_o=\left\|W_r^\top W_r-I_{d \times d}\right\|_2^2.\notag
\end{equation}
The attention coefficient is obtained by the following equation, where $\bm{v}^\top$ is the attention vector:
\begin{equation}
	\label{attention}
	\alpha^{l}_{ijk} = \frac{\exp(\bm{v}^\top {\bm{h}_{r_k}^{\to e_j}})}{\sum_{{e_m\in \mathcal{N}_{e_i}},{<e_i, r_n, e_m> \in T_s \cup T_t}} \exp(\bm{v}^\top \bm{h}_{r_n}^{\to e_m})}
\end{equation}

\textbf{Intra- \& Cross-Graph Representation Learning.}
Based on the above, we can express the embedding of $e_i$ at the $(l+1)$-th layer $\bm{h}_{e_i}^{l+1}$ as Eq.~\ref{intra}, where $W^{l+1}$ is specified as $W^{l+1}=I_{d \times d}-2\bm{h}_{r_k}\bm{h}_{r_k}^{\top}$ by the triplet $<e_i, r_k, e_j>$ inclusive relation embedding $\bm{h}_{r_k}$. We adopt the $\tanh(\cdot)$ as the activation function. The Householder transformation $W^{l+1}$ is applied on the last layer embedding $\bm{h}_{e_i}^{l}$ to restore the useful relative positions of KG entities at each layer recursively. 

Overall, the \textit{intra-graph representation} $\bm{h}_{e_i}$ of $e_i$ is obtained by concatenating embeddings from all layers. Its \textit{cross-graph representation} $\bm{h}_{e_i}^{proxy}$ can be described by $\bm{h}_{e_i}$ and proxy $\bm{q}_j$, where the latter is generated by \emph{Proxy Matching Attention Layer} \cite{mao2021boosting} to align the embeddings across two graphs. With $S_{p}$ representing the set of proxy vectors, and $\textrm{sim}(\cdot)$ denoting the cosine similarity, we have:
\begin{equation}
	\bm{h}_{e_i} = [\bm{h}^0_{e_i}\|\bm{h}^1_{e_i}\|...\|\bm{h}^l_{e_i}]\quad\text{ and }\quad
	\bm{h}_{e_i}^{proxy} = \sum_{\bm{q}_j \in S_{p}} \frac{\exp(\textrm{sim}(\bm{h}_{e_i},\bm{q}_j))} {\sum_{\bm{q}_{k} \in S_{p}} \exp(\textrm{sim}(\bm{h}_{e_i},\bm{q}_{k}))}(\bm{h}_{e_i} - \bm{q}_{j}).\notag
\end{equation}
Finally, we employ a gating mechanism \cite{srivastava2015highway} to integrate both intra-graph representation $\bm{h}_{e_i}$ and cross-graph representation $\bm{h}^{proxy}_{e_i}$ into $\bm{h}^{f}_{e_i}$:
\begin{equation}
	\bm{\theta}_{e_i} = \textrm{sigmoid}(\bm{W}_{g}\bm{h}^{proxy}_{e_i}+\bm{b}),\quad\quad
	\bm{h}_{e_i}^{f} = [(\bm{\theta}_{e_i}\cdot \bm{h}_{e_i} + (1-\bm{\theta}_{e_i})\cdot \bm{h}^{proxy}_{e_i}) \|r_{e_i}],\notag
\end{equation}
where $\bm{W}_{g}$ and $\bm{b}$ are the gate weight and gate bias, respectively. The learnable weight of $e_i$ is also attached to the embedding. It is worth noticing that for each entity on either $G_s$ or $G_t$, they are encoded by one shared KEESA with below spectral contrastive learning for a unified representation space.

\subsection{Spectral Contrastive Learning}
In this part, we propose the spectral contrastive learning loss $\mathcal{L}_{\mathrm{info}}$ with high-quality negative sample mining, which serves both tasks (entity alignment and dangling detection) at the same time. Specifically, given a pre-aligned matchable entity $e_i\in \mathcal{X}_p$, let there be a paired positive sample entity $e^i_{+} \in \mathcal{X}_p$, such that $(e_i, e^i_{+}) \in S$, and $N$ sampled entity $\{e^i_{j}\}^{N}$ as negative samples $(e_i, e^i_{j}) \notin S$. The spectral contrastive learning loss is one specific form of alignment loss $H(\cdot)$:
\begin{equation}
	\mathcal{L}_{\mathrm{info}} = \sum_{e_i\in \mathcal{X}_p} \log\left[1 + \sum^N_{j} \exp(\lambda~ H(e_i, e^i_{+}, e^i_{j}))\right].
	\label{L1}
\end{equation}

\textbf{Unified Representation for Entity Alignment.} We expect a unified feature space where the distance between aligned anchor node pairs is as close as possible, while the unaligned is on the contrary. To satisfy this intuition, we introduce an alignment loss:
\begin{equation}
	H(e_i, e^i_{+}, e^i_j) = [ \textrm{sim}(e_i, e^i_{j}) - \textrm{sim}(e_i, e^i_{+}) + \gamma ]_+,
	\label{eq:hh}
\end{equation}
where $[x]_+$ represents $\max(0,x)$ and $\textrm{sim}(\cdot)$ indicates $L_2$-norm distance between the embeddings. A soft margin $\gamma$ is involved to discourage trivial solutions, e.g., $\textrm{sim}(e_i, e^{i}_{j}) = \textrm{sim}(e_i, e^{i}_{+}) = 0$.

\textbf{Discrimination for Dangling Detection.} For our proposed dangling detection, the vital task is to discriminate the dangling from the matchable ones with unlabeled dangling entities. Hence unsupervised method of spectral clustering is exploited to separate two types of entities. We design the loss function according to Lemma~\ref{spec_clus} to achieve its equivalent effect.
\begin{lemma}
	\label{spec_clus}
	Given one positive sample $p^{+}$ for $q$, and $N$ negative samples $\{p^{-}_j\}^N$ (node set: $\{q, p^{+}\} \cup \{p^{-}_j\}^N$), employing the following loss function is equivalent to conducting spectral clustering on similarity graph $\pi$ with the temperature hyper-parameter $\lambda$:
	\begin{equation}
		\begin{aligned}
			\textit{infoNCE}(q, p^{+}, \{p^{-}\}^{N}) &= - \log \frac{\exp(\lambda~ \textrm{sim}(q, p^{+}))}{ \exp(\lambda~ \textrm{sim}(q, p^{+}))+ \sum_{j}^N \exp(\lambda~ \textrm{sim}(q, p^{-}_j))}.\\
		\end{aligned}
	\end{equation}
\end{lemma}
The equivalence has been discussed in previous studies \cite{tan2023contrastive, van2022probabilistic, chen2020simple, sun2020circle}. Regarding our proposed problem, the positive samples are the corresponding pairs whereas the negative samples are those sampled unaligned pairs. The equivalence is derived as follows: 
\begin{equation}
	\begin{aligned}
		\textit{infoNCE}(q, p^{+}, \{p^{-}\}^{N}) &=\log\frac{\exp(\lambda\text{sim}(q, p^{+}))+\sum_{j}^N\exp(\lambda\text{sim}(q, p^{-}_j))}{\exp(\lambda~ \textrm{sim}(q, p^{+}))}\\
		&=\log [1+ \frac{ \sum_{j}^N \exp(\lambda~ \textrm{sim}(q, p^{-}_j))}{\exp(\lambda~ \textrm{sim}(q, p^{+}))}]\\
		&=\log\left[1 + \sum^N_{j} \exp(\lambda~ \textrm{sim}(q, p^{-}_j)-\lambda~ \textrm{sim}(q, p^{+}))\right].\notag
	\end{aligned}
\end{equation}
The spectral contrastive learning loss could be obtained by replacing the exponent term with Eq.~\ref{eq:hh}.

\noindent\textbf{Remark.} In the alignment loss function, we observe that dangling entities are only in the negative samples, and the entities in the positive samples are all matchable. Such a stark asymmetry provides an advantage in distinguishing between dangling and matchable entities.


We also prove that $\mathcal{L}_{\mathrm{info}}$ (Eq.~\ref{L1}) can promote model learning by Lemma~\ref{spec_hard} (see appendix~\ref{app_5} for proof).
\begin{lemma}
	\textit{The loss $\mathcal{L}_{\mathrm{info}}$ can mine high-quality negative samples, which we show has an equivalent effect to \textit{truncated uniform negative sampling (TUNS)} in \cite{sun2018bootstrapping}.}
	\label{spec_hard}
\end{lemma}
Minimizing the spectral contrastive loss of Eq.~\eqref{L1} maps matchable and dangling entities into a unified but distinguishable feature space for improved entity alignment while facilitating dangling detection. In practice, we adopt the loss normalization trick \cite{gao2022clusterea} on $H(\cdot)$ to speed up training.

\section{Iterative Positive-Unlabeled Learning for Dangling Detection}\label{sec:iPULE}
We expect to avoid any additional computational overhead for EA if few entities are matchable for the source and target KG. Thus, we address a more challenging problem in EA: given partial pre-aligned matchable entities as positive samples (i.e., $\mathcal{X}_p$), how to jointly predict the proportion of matchable entities in the unlabeled nodes (i.e., $\pi^\mathrm{u}_{\mathrm{p}}$) and identify those entities? If $\pi^\mathrm{u}_{\mathrm{p}}$ could be predicted, it could serve as an indicator whether we should proceed to EA. We propose to address the problem by PU learning.

\textbf{Unbiased Risk Estimator.}
First, we propose a new unbiased estimation for PU learning without any constraint (i.e., $p_\mathrm{u}(x) = p(x)$ in \cite{kiryo2017positive, niu2016theoretical}) on unlabeled samples distribution $p_\mathrm{u}(x)$ concerning the overall distribution $p(x)$. Assuming that $\pi_{\mathrm{p}}$ and $\pi^\mathrm{u}_{\mathrm{p}}$ are known (estimation strategy would be given later), we have:
\begin{theorem}
	Suppose that $g \in \mathcal{G}:\mathbb{R}^d\to\mathbb{R}$ is a binary classifier, and $\ell:\mathbb{R}\times\{\pm1\}\to\mathbb{R}$ is the loss function by which $\ell(t,y)$ means the loss incurred by predicting an output $t$ when the ground truth is $y$. $\widehat{R}_{\mathrm{pu}}(g)$ is the \textbf{unbiased risk estimator} of $R(g)$:
	\begin{equation}
		\widehat{R}_{\mathrm{pu}}(g)= \pi_{\mathrm{p}} \widehat{R}_{\mathrm{p}}^+(g)+ \frac{\pi_\mathrm{n}}{\pi^\mathrm{u}_\mathrm{n}} \cdot \left[ \widehat{R}_{\mathrm{u}}^-(g)-\pi^\mathrm{u}_{\mathrm{p}}\widehat{R}_{\mathrm{p}}^-(g) \right],
		\label{Rpu}
	\end{equation}
	where $\pi_\mathrm{n} = p(y=-1)$ and $\pi^\mathrm{u}_\mathrm{n}=p_\mathrm{u}(y=-1)$ are estimable class priors given $\pi_{\mathrm{p}}$ and $\pi^\mathrm{u}_{\mathrm{p}}$, $R_p^+(g)= \mathbb{E}_{X\sim p_\mathrm{p}(x)} [\ell(g(X),+1)]$ and $R_n^-(g)=\mathbb{E}_{X\sim p_\mathrm{n}(x)}[\ell(g(X),-1)]$ (\text{see appendix~\ref{app_2} for proof}).  
	\label{theorem_unbiased}
\end{theorem}

By our proof, $\widehat{R}_{\mathrm{pu}}(g)$ is an unbiased risk bound for the PU learning. More importantly, the bound provided by Thm.~\ref{deviation_bound} is a tighter uniform deviation bound than the classic \textit{Non-negative Risk Estimator} \cite{kiryo2017positive}:
\newtheorem{thm}{Theorem}
\begin{theorem}
	Let $\mathrm{Var}(R)$ denote the uniform deviation bound of risk estimator $R$, and Non-negative Risk Estimator be $\widehat{R'}_{\mathrm{pu}}(g)$, then:$\mathrm{Var}(\widehat{R}_{\mathrm{pu}}(g)) \,< \, \mathrm{Var}(\widehat{R'}_{\mathrm{pu}}(g)) (\text{see appendix~\ref{app_3} for proof}).\notag$
	\label{deviation_bound}
\end{theorem}

\textbf{Positive Unlabeled Loss Function.}
Since it is evident that all negative samples exist in unlabeled data, we have $\frac{\pi_\mathrm{n}}{\pi^\mathrm{u}_\mathrm{n}} < 1$ and thus we apply a hyper-parameter $\alpha=\frac{\pi^\mathrm{u}_\mathrm{n}}{\pi_\mathrm{n}}$ to scale $\pi_{\mathrm{p}} \widehat{R}_{\mathrm{p}}^+(g)$ equivalently and $\max(\cdot)$ to restrict the estimated $\pi_\mathrm{n} R_\mathrm{n}^-(g) \ge 0$. The PU learning loss function is formulated as: 
\begin{equation}
	\mathcal{L}_{\mathrm{pu}}=\alpha\pi_{\mathrm{p}} \widehat{R}_{\mathrm{p}}^+(g)+ \max\{0,\widehat{R}_{\mathrm{u}}^-(g)-\pi^\mathrm{u}_{\mathrm{p}}\widehat{R}_{\mathrm{p}}^-(g) \}\\,
	\notag
	\label{pnrisk}
\end{equation}
We specify the corresponding risk function using cross-entropy losses as below:
\begin{equation}
	\begin{aligned}
		\widehat{R}_{\mathrm{p}}^+(g)= \frac{1}{|\mathcal{X}_p|}\sum_{e_i \in \mathcal{X}_p} \log\hat{y}_{i}(+1),
		\widehat{R}_{\mathrm{u}}^-(g)= \frac{1}{|\mathcal{X}_u|}\sum_{e_i \in \mathcal{X}_u} \log\hat{y}_{i}(-1),
		\widehat{R}_{\mathrm{p}}^-(g)= \frac{1}{|\mathcal{X}_p|}\sum_{e_i \in \mathcal{X}_p} \log\hat{y}_{i}(-1)
		\notag     
		\label{p_risk}
	\end{aligned}
\end{equation}
where the output logit for $e_i \in E_s \cup E_t$, being labeled as a state $u \in \{+1,-1\}$, is $\hat{y}_{i}(u) = softmax(\mathbf{MLP}(h_{e_i}^f))$, based on KEESA output $h_{e_i}^f$. Hence each term in the final loss can be calculated or estimated without the negative labels.

\textbf{Iterative PU Learning with Prior Estimator}
\label{sec:EM-PU_chapter}
How could we estimate prior $\pi_{\mathrm{p}}$ and $\pi^\mathrm{u}_{\mathrm{p}}$? Inspired by \cite{jiang2024learning}, we introduce a hidden variable in the model as well as an iterative approach. We adopt a variational approximation strategy and a warm-up phase to tackle the cold start problem, as shown in Alg.~\ref{EM-PU}. First, we estimate and fix the class prior $\pi_{\mathrm{p}}$ and $\pi^\mathrm{u}_{\mathrm{p}}$ by the ratio of the anchor points in the training set. $\mathcal{L}_{\mathrm{info}}$ is optimized together with $\mathcal{L}_{\mathrm{pu}}$ for a discriminative embedding space in the warm-up phase. Finally, we minimize $\mathcal{L}_{\mathrm{pu}}$ to update the class prior and the model parameters alternately till convergence.

The convergence guarantee is provided in Thm.~\ref{theorem_Q_Rpu}, which mostly follows the convergence of EM algorithm. We collect the proof in appendix~\ref{app_1}.

\begin{algorithm}
	\caption{iPULE (\textbf{i}terative \textbf{PU} \textbf{L}earning with Prior \textbf{E}stimator)}
	\begin{algorithmic}[1]
		\REQUIRE
		$G_s$ and $G_t$ are treated as one input graph $G=(\mathcal{V},\mathcal{E})$, positive-node set $\mathcal{P}=\mathcal{X}_p$, unlabeled-node set $\mathcal{U}=\mathcal{X}_u$, classifier $f$ with initial parameters $\theta_0$, KEESA $\text{Enc}(G, \psi)$ with initial parameters $\psi_0$ and warm-up epoch $N$. $\mathcal{L}$ represents training loss. 
		\ENSURE Model parameters $\theta, \psi$ and estimated prior $\hat{\pi}_\mathrm{p}$ and $\hat{\pi}_\mathrm{p}^\mathrm{u}$
		\STATE $l\leftarrow\infty,\quad\hat{\pi}_\mathrm{p}^\mathrm{u} \leftarrow \hat{\pi}_\mathrm{p}\leftarrow \frac{|\mathcal{P}|}{|\mathcal{P}|+|\mathcal{U}|},\quad i \leftarrow 0,\quad\beta =\beta_0;${\hfill //Initial value}
		\STATE $\mathcal{L} \leftarrow \beta \cdot \mathcal{L}_\mathrm{info}+(1-\beta)\cdot \mathcal{L}_\mathrm{pu};${\hfill //Loss of warm-up}
		\REPEAT
		\STATE $\mathbf{X} \leftarrow \text{Enc}(G,\psi);$ {\hfill //Entity embedding matrix \textbf{X}}
		\STATE $\theta, \psi \leftarrow \arg\min_{\theta, \psi}\mathcal{L}(\theta;\mathbf{X},\mathbf{y},\mathcal{P},\mathcal{U});${\hfill //Optimize Enc($\cdot$) and $f$ jointly}
		\STATE $l\leftarrow\mathcal{L}(\theta;\mathbf{X},\mathbf{y},\mathcal{P},\mathcal{U});$\\
		\UNTIL{N epochs is over}{\hfill //Warm-up phase to solve cold start}
		\STATE $\mathcal{L} \leftarrow \mathcal{L}_\mathrm{pu};$
		\REPEAT
		\STATE $\mathbf{X} \leftarrow \text{Enc}(G,\psi), \quad \hat{y}_{i}\leftarrow f(\mathbf{X},i;\theta) \,\,\mathrm{ for\,\, all \,\,} i\in\mathcal{V};$ 
		\STATE $\hat{\pi}_{\mathrm{p}}^{\mathrm{u}}\leftarrow|\mathcal{U}|^{-1}\sum_{i\in\mathcal{U}}\mathbb{I}[\hat{y}_{i}(+1)>0.5],\quad \hat{\pi}_{\mathrm{p}}\leftarrow \frac{|\mathcal{P}|+|\mathcal{U}| \cdot \hat{\pi}_{\mathrm{p}}^{\mathrm{u}}}{|\mathcal{P}|+|\mathcal{U}|};${\hfill //E step}\\
		\STATE $l^{\prime} \leftarrow l,\quad l\leftarrow\mathcal{L}(\theta;\mathbf{X},\mathbf{y},\mathcal{P},\mathcal{U});$\\
		\STATE $\theta, \psi \leftarrow \arg\max_{\theta, \psi}-\mathcal{L}(\theta;\mathbf{X},\mathbf{y},\mathcal{P},\mathcal{U});${\hfill //M step}
		\UNTIL{$| l^{\prime} - l |$ converge \textbf{OR} $\hat{\pi}_{\mathrm{p}}$ converge}
		\RETURN 
	\end{algorithmic}
	\label{EM-PU}
\end{algorithm}

\begin{theorem}
	\textit{Given the assumptions of marginalization in Eq.~\ref{z_margin} and Eq.~\ref{y_z_margin}, the objective function of $-\mathcal{L}_{\mathrm{pu}}$ is the same as the expectation function Q of Eq.~\ref{Q_func} where the loss function is the cross entropy function $CE(\bar{y}_{i}, \hat{y}_{i}) = -\bar{y}_{i}(+1)\log\hat{y}_{i}(+1) - \bar{y}_i(-1)\log\hat{y}_i(-1)$ on the \textbf{preference condition}: $\sum_{j\in\mathcal{U}}\frac{1}{|\mathcal{U}|}\log\frac{\hat{y}_j(+1)}{\hat{y}_j(-1)} \approx \sum_{i\in\mathcal{P}}\frac{1}{|\mathcal{P}|}\log\frac{\hat{y}_i(+1)}{\hat{y}_i(-1)}$.}
	\label{theorem_Q_Rpu}
\end{theorem}
The iterative process of our method is a special case of the EM algorithm. We hold the same assumptions as the EM algorithm and we further assume the training of $f$ is able to find the globally optimal $\theta$. Although the assumptions seem to be too strict, the algorithm typically converges in practice as we verified in the experimental section.

\section{Experiments}\label{exp_begin}
Our method is evaluated on datasets GA16K, DBP2.0, and GA-DBP15K. DBP2.0 and GA-DBP15K are used for the verification of iPULE. To address incomparability caused by inconsistent metrics, we adopt the GA16K dataset to enable compromised comparison of the Dangling-Entities-Unaware EA method. We further compare our method with dangling aware baselines on DBP2.0. \textit{Statistics of experimental dataset} in appendix~\ref{app_7}, and \textit{additional experiment} in appendix~\ref{app_8}. 

\textbf{Datasets.} The training/test sets for each dataset are generated using a fixed random seed. For entity alignment, 30\% of matchable entity pairs constitute the training set, while the remaining form the test set. For dangling entity detection, we did not utilize any labeled dangling entity data, in contrast to prior work which labels an extra 30\% of the dangling entities for training \cite{sun2021knowing, liu2022dangling}.

\textbf{Baselines.} Since our work does not take advantage of any side information, we emphasize its comparison with the previous methods purely depending on graph structures. These works majorly incorporate two types:

\textit{Dangling-Entities-Unaware.} We include advanced entity alignment methods in recent years: GCN-Align \cite{wang2018cross}, RSNs \cite{guo2019learning}, MuGNN \cite{cao2019multi}, KECG \cite{li2019semi}. Methods with bootstrapping to generate semi-supervised structure data are also adopted: BootEA \cite{sun2018bootstrapping}, TransEdge \cite{sun2019transedge}, MRAEA \cite{mao2020mraea}, AliNet \cite{sun2020knowledge}, and Dual-AMN \cite{mao2021boosting}.

\textit{Dangling-Entities-Aware.} To the best of our knowledge, the method of \cite{sun2021knowing} is the most fairly comparable baseline which is based on MTransE \cite{chen2017multilingual} and AliNet \cite{sun2020knowledge}. Because MHP \cite{liu2022dangling} over-emphasized more use of labeled dangling data like high-order similarity information which is also based on the above two methods, while SoTead \cite{luo2021graph} and UED \cite{luo2022accurate} utilize additional side-information. SoTead \cite{luo2021graph} and UED \cite{luo2022accurate} can only execute the degraded version on DBP2.0 cause no side-information is available on that. We exclude them from baselines for our methods. \cite{sun2021knowing} introduces three techniques to address the dangling entity issue: nearest neighbor (NN) classification, marginal ranking (MR), and background ranking (BR).

\textbf{Implementation Detail.}
We use the {\tt Keras} framework for developing our approach. Our experiments are conducted on a workstation with an NVIDIA Tesla A100 GPU, and 80GB memory.

By default, the embedding dimension is set to 128 with the depth of GNN set to 2 and a dropout rate of 0.3. A total of 64 proxy vectors are used and margin $\gamma$ = 1. RMSprop optimizer is adopted with a learning rate of 0.005 and batch size of 5120. $\lambda$ is set to be 30. $\beta$ is set as 1e-3 for all datasets. CSLS \cite{lample2018word} is adopted as the distance metric for alignment. As we found, the $\tanh$ function changes rapidly in the region close to $0$ but stays stable in the region beyond $[-3, 3]$. Hence we initialize the $r_{e_{j}}$ to $1$ to prevent gradients oscillation or near-zero gradients.

\subsection{Experiments of iPULE Convergence and Class Prior Estimation}\label{sec:IPUPE} 
DBP2.0 $\pi_{\mathrm{p}}$ between 20\%-50\% contains more entities to be aligned, trained by pre-aligned 9\%-15\% nodes then judged by iPULE as aligned KGs. GA-DBP15K $\pi_{\mathrm{p}}$ between 10\%-25\% are treated as unaligned KGs ignoring the pre-aligned part, trained by all pre-aligned 10\%-25\% nodes. We get accurate estimation and convergence results as shown in Fig.~\ref{fig:visual_EM_PU_addition}. As iPULE progresses, the estimated class prior gradually approaches the true value as the first row for GA-DBP15K and the second for DBP2.0. $\pi_\mathrm{p}^\mathrm{u}=0$ for GA-DBP15K while DBP2.0's are given by red dotted line respectively. The $\pi_{\mathrm{p}}$ for GA-DBP15K is stably consistent as pre-aligned proportion due to accurate estimation of its $\pi_\mathrm{p}^\mathrm{u}$. As common in PU learning \cite{yoo2022graph}, iPULE treats more nodes as positive when $\pi_{\mathrm{p}}\approx 50\%$ in FR-EN.
\begin{figure}[ht]
	\centering
	\begin{minipage}[b]{0.96\textwidth}
		\centering
		\includegraphics[width=0.8\textwidth]{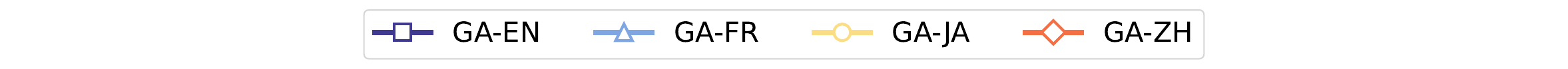}
	\end{minipage}
	\begin{minipage}[b]{0.24\textwidth}
		\centering
		\includegraphics[height = 0.13\textheight,width=\textwidth]{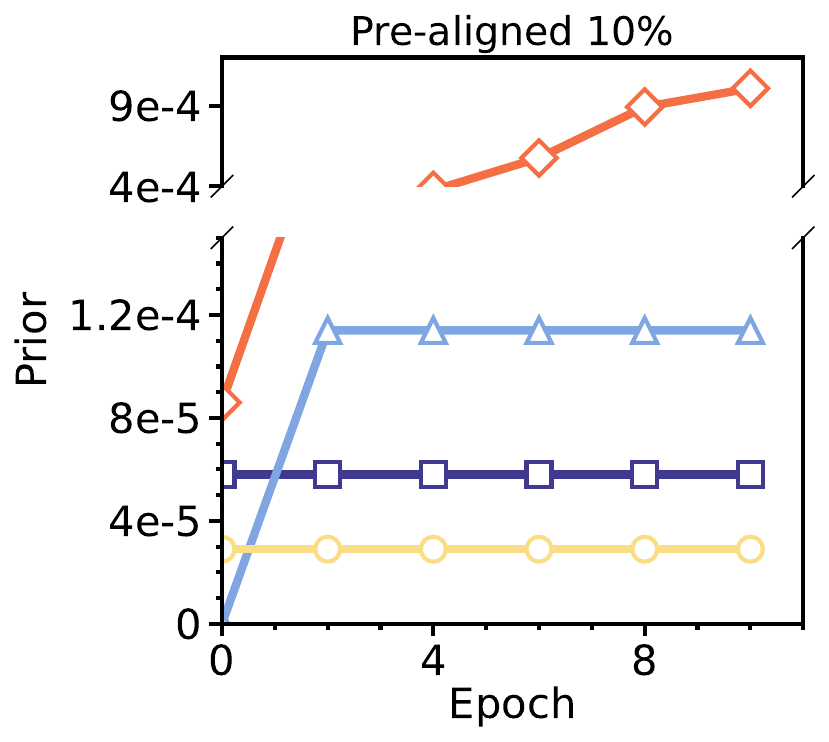}
	\end{minipage}
	\hfill
	\begin{minipage}[b]{0.24\textwidth}
		\centering
		\includegraphics[height = 0.13\textheight,width=\textwidth]{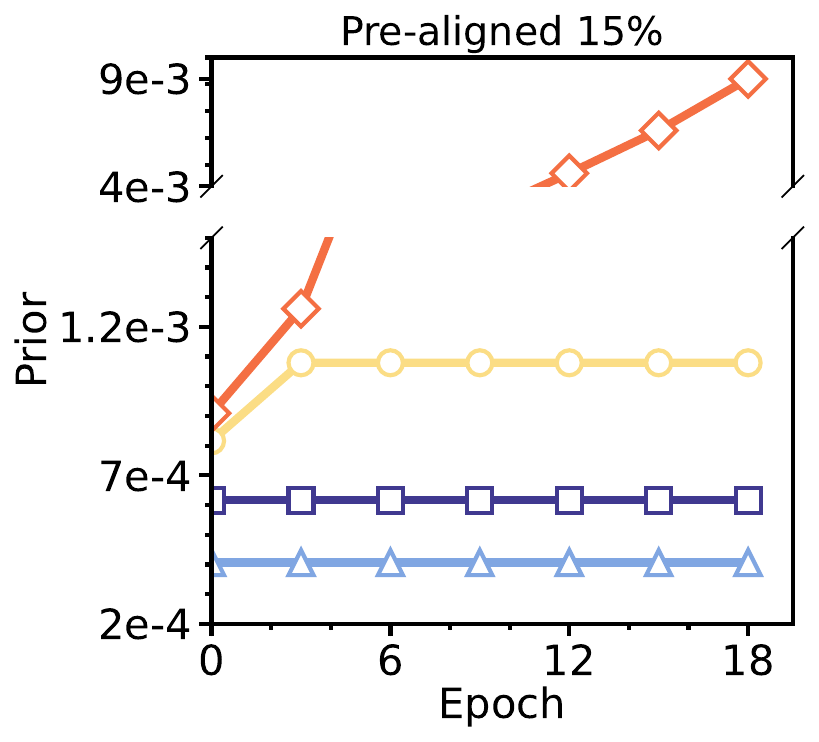}
	\end{minipage}
	\hfill
	\begin{minipage}[b]{0.23\textwidth}
		\centering
		\includegraphics[height = 0.13\textheight,width=\textwidth]{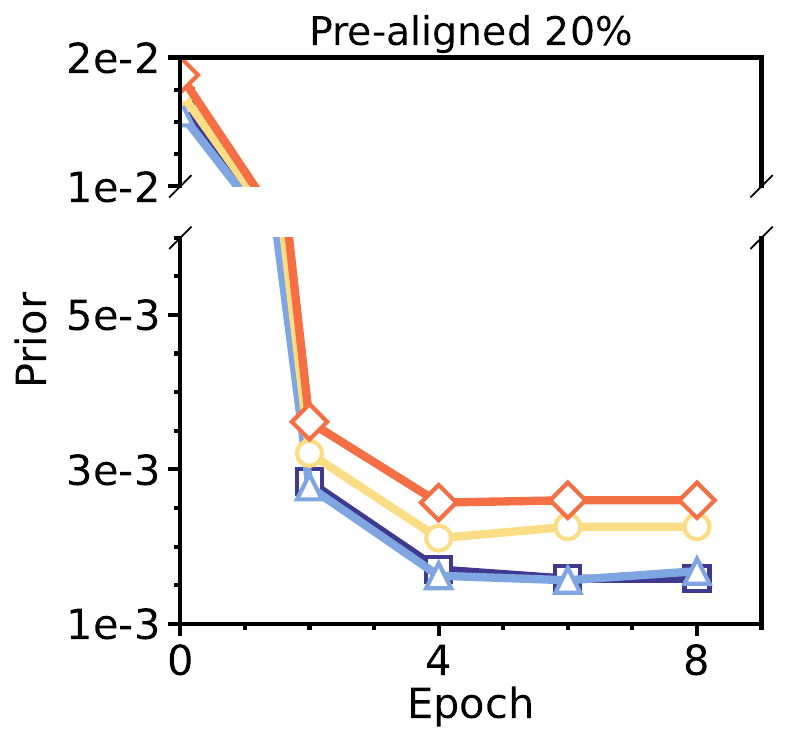}
	\end{minipage}
	\hfill
	\begin{minipage}[b]{0.23\textwidth}
		\centering
		\includegraphics[height = 0.13\textheight,width=\textwidth]{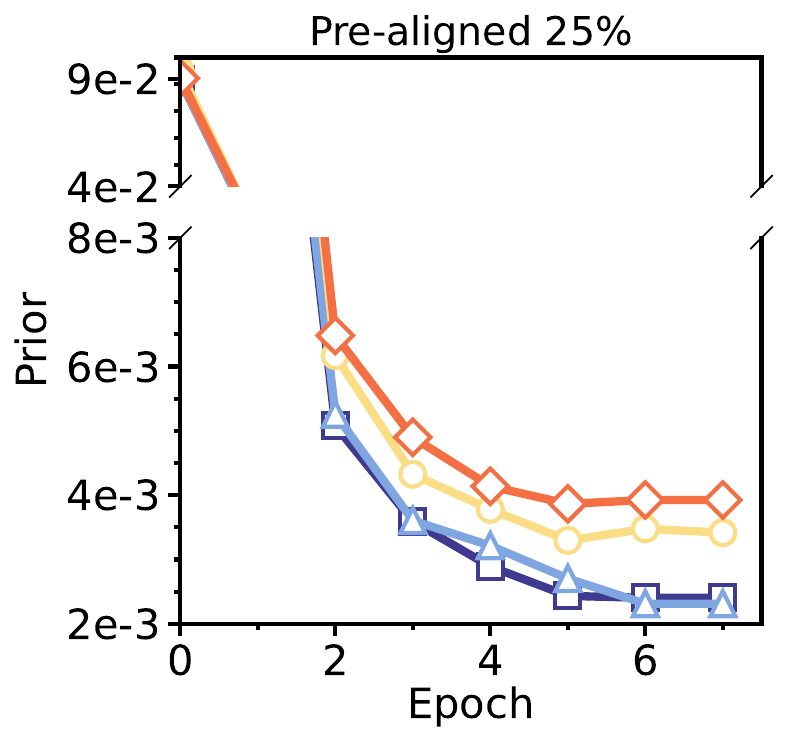}
	\end{minipage}
	
	\begin{minipage}[b]{0.96\textwidth}
		\centering
		\includegraphics[width=0.8\textwidth]{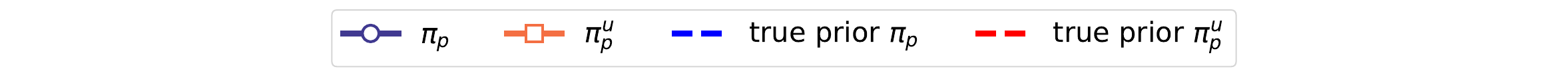}
	\end{minipage}
	\begin{minipage}[b]{0.24\textwidth}
		\centering
		\includegraphics[height = 0.13\textheight,width=\textwidth]{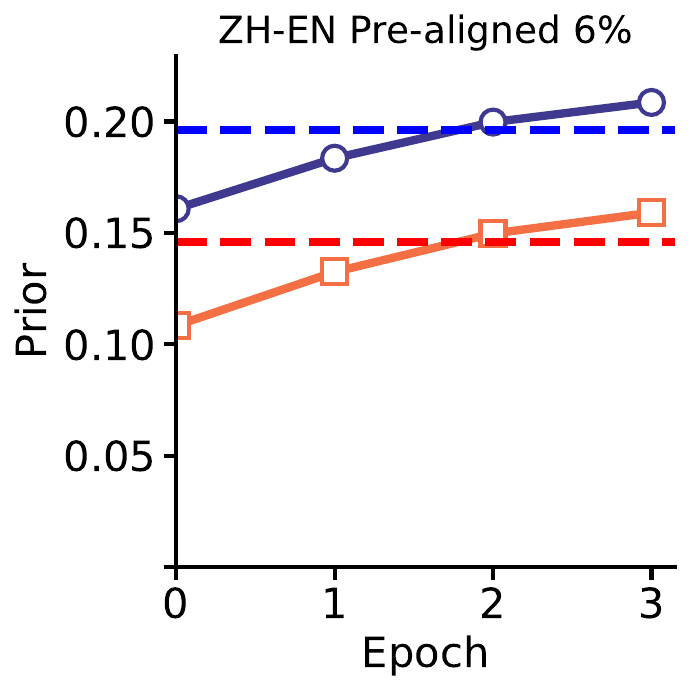}
	\end{minipage}
	\hfill
	\begin{minipage}[b]{0.24\textwidth}
		\centering
		\includegraphics[height = 0.13\textheight,width=\textwidth]{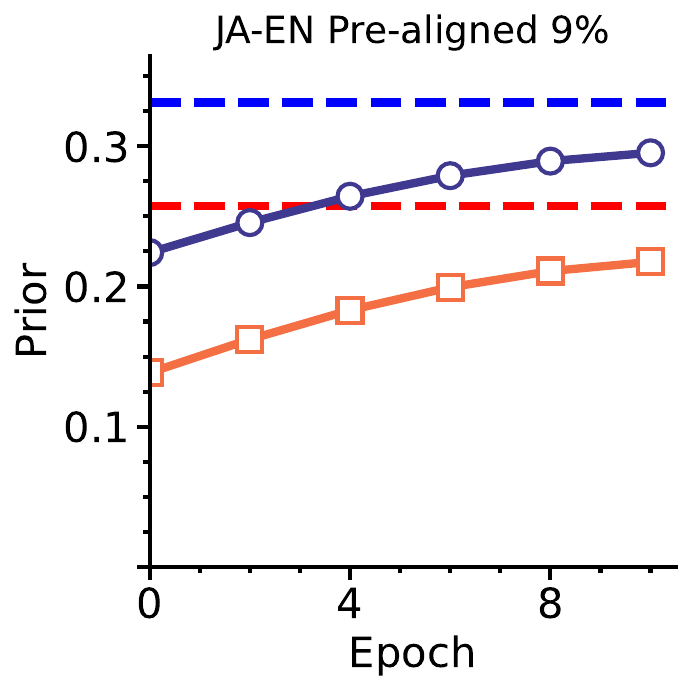}
	\end{minipage}
	\hfill
	\begin{minipage}[b]{0.24\textwidth}
		\centering
		\includegraphics[height = 0.13\textheight,width=\textwidth]{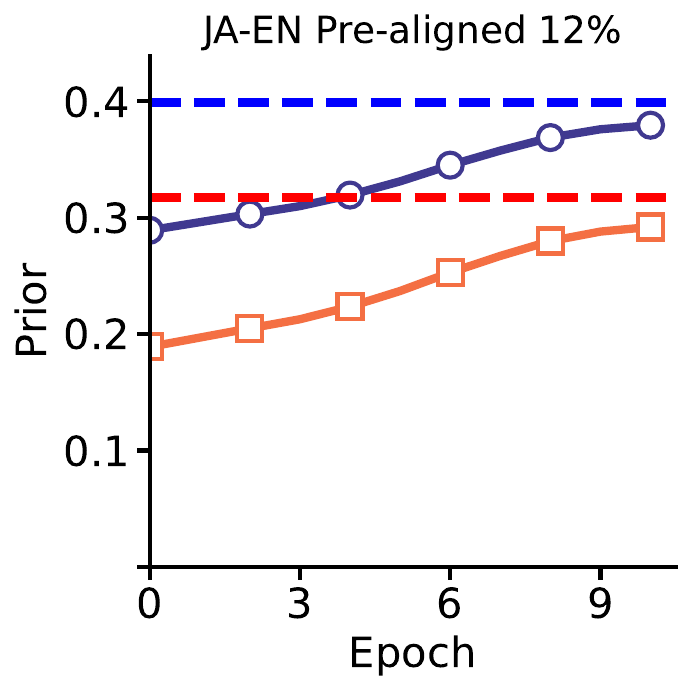}
	\end{minipage}
	\hfill
	\begin{minipage}[b]{0.24\textwidth}
		\centering
		\includegraphics[height = 0.13\textheight,width=\textwidth]{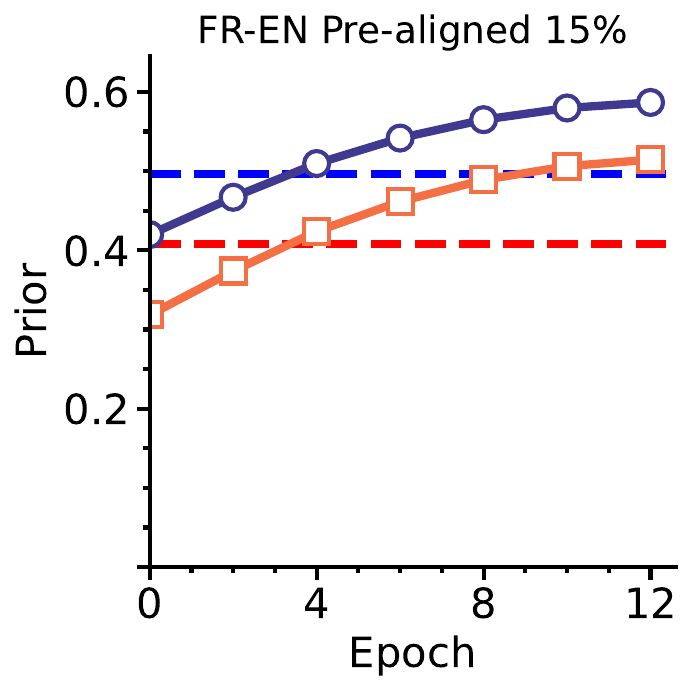}
	\end{minipage}
	\caption{Prior estimation GA-DBP15K and DBP2.0. (loss convergence in appendix~\ref{app_7}).}
	\label{fig:visual_EM_PU_addition}
\end{figure}

\subsection{Experiments Unaware of Dangling Entities}\label{sec:unaware}
We show the experiments on baselines without considering dangling entities in this section. 

\textbf{Dangling-Entities-Unaware Baselines Comparison.} 
The direct comparison between our method and the dangling-entities-unaware baselines is unavailable due to inconsistent metrics used. Hence, we adopt the GA16K dataset as a compromise and do not remove any detected dangling entities for entity alignment. Thus the ranking list $S$ in Hits@K only contains (matchable) entities in the source graph since GA16K only contains dangling entities in the target KG. In Tab.~\ref{tab:GA16K_ent_alignment}, Dual-AMN demonstrates a competitive performance but is inferior to ours at Hits@1. MRAEA performs similarly to Dual-AMN since the latter is built on the former. TransEdge performs poorly since the method adopts semi-supervised bootstrapping to mine anchor entities iteratively. The presence of dangling entities could lead to false anchors and spread of error. Meanwhile, it is also a relation-centric approach that suffers from insufficient relation information on GA16K. Other baselines exhibit up-to-par performance but our method delivers consistently superior or state-of-the-art Hits@Ks.

\begin{figure}[htbp]
	\begin{minipage}{0.49\linewidth}
		\begin{minipage}{0.49\linewidth}
			\centering
			\includegraphics[height = 0.11\textheight, width = 0.95\textwidth]{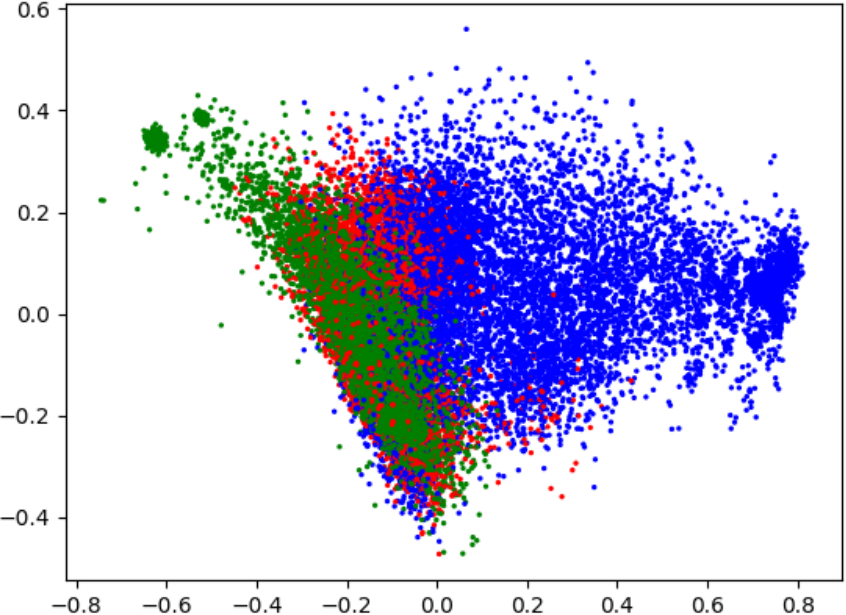}
			\put(-82,63){\textbf{a}}
			\label{a}
		\end{minipage}
		\begin{minipage}{0.5\linewidth}
			\centering
			\includegraphics[height = 0.11\textheight, width = 0.95\textwidth]{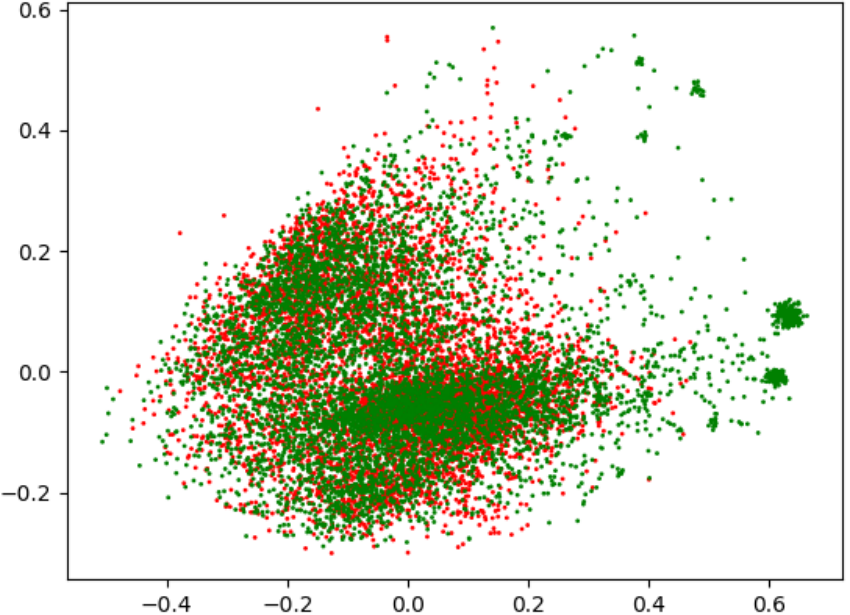}
			\put(-82,63){\textbf{b}}
			\label{b}
		\end{minipage}
		\begin{minipage}{0.49\linewidth}
			\centering
			\includegraphics[height = 0.11\textheight, width = 0.95\textwidth]{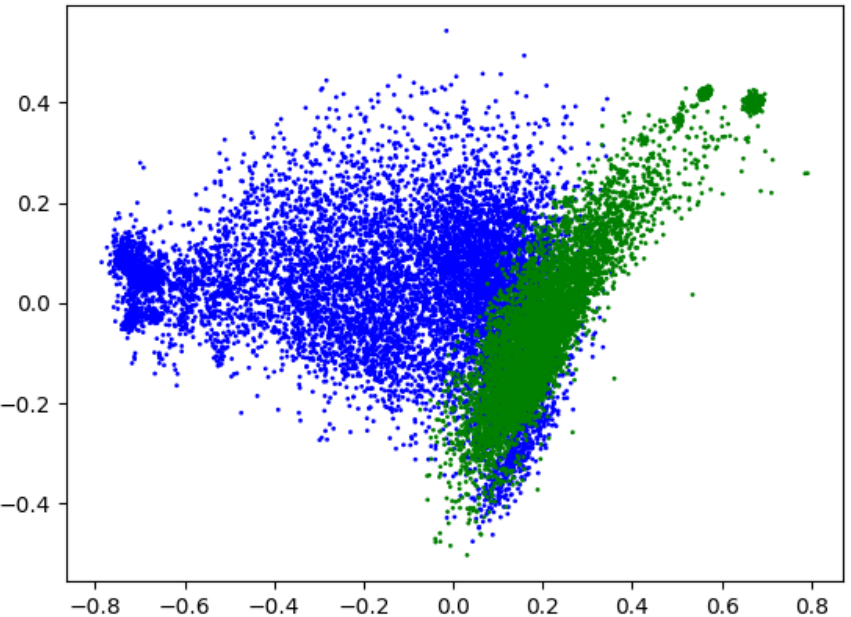}
			\put(-82,63){\textbf{c}}
			\label{c}
		\end{minipage}
		\begin{minipage}{0.5\linewidth}
			\centering
			\includegraphics[height = 0.11\textheight, width = 0.95\textwidth]{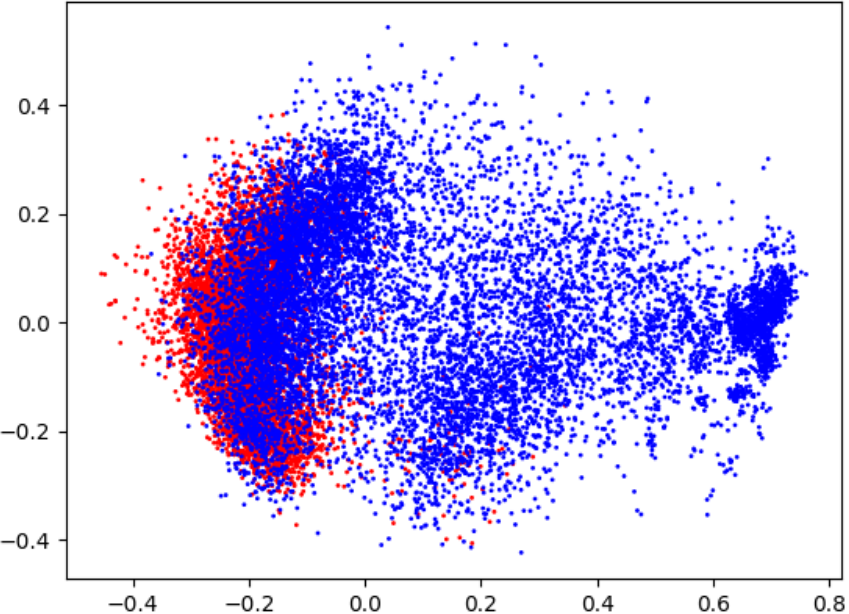}
			\put(-82,63){\textbf{d}}
			\label{d}
		\end{minipage}
		\caption{Visualization of entity representations learned by our method on GA16K dataset.}
		\label{fig:visual}
	\end{minipage}
	\begin{minipage}{0.5\linewidth}
		\centering
		\renewcommand\arraystretch{1.045}
		\setlength{\tabcolsep}{6pt}
		\begin{tabular}{llll}
			\toprule
			\multirow{2}{*}{Method} & \multicolumn{3}{c}{$\rm{GA16K}$}\\
			& H@1 & H@10 & H@50\\
			\hline
			BootEA & 13.95 & 37.25 & 49.08  \\
			TransEdge & 0.03 & 0.12 & 0.14   \\
			MRAEA & 63.97 & 76.64 & 81.06 \\
			GCN-Align & 29.48  & 45.64 & 57.15  \\
			RSNs&  9.40 & 42.70 & 46.70 \\
			MuGNN &  62.17 & 76.25 & 80.87 \\
			KECG & 44.18 & 57.73 & 63.41  \\
			AliNet& 48.53 & 67.72 & 74.50  \\
			Dual-AMN& \underline{64.49} & \textbf{80.55} & \textbf{84.67} \\
			Ours& \textbf{67.59} & \underline{80.33} & \underline{84.35} \\
			\bottomrule
		\end{tabular}
		\captionof{table}{Performance comparison with dangling-entities-unaware baselines on GA16K.}
		\label{tab:GA16K_ent_alignment}
	\end{minipage}  
\end{figure}
\begin{table*}[hbtp]
	\centering
	\renewcommand\arraystretch{0.8}
	\resizebox{.99\textwidth}{!}{\small
		{
			\small
			\setlength{\tabcolsep}{3pt}
			\begin{tabular}{llcccccccccccccccccc}
				\toprule
				\multicolumn{2}{c}{\multirow{2}{*}{Methods}} &
				\multicolumn{3}{c}{ZH-EN} & \multicolumn{3}{c}{EN-ZH} & \multicolumn{3}{c}{JA-EN} & \multicolumn{3}{c}{EN-JA} &  \multicolumn{3}{c}{FR-EN} & \multicolumn{3}{c}{EN-FR}\\
				\cmidrule(lr){3-5} \cmidrule(lr){6-8} \cmidrule(lr){9-11} \cmidrule(lr){12-14} \cmidrule(lr){15-17} \cmidrule(lr){18-20}
				&& Prec. & Rec. & F1 & Prec. & Rec. & F1 & Prec. & Rec. & F1 & Prec. & Rec. & F1 & Prec. & Rec. & F1 & Prec. & Rec. & F1 \\ 
				\midrule
				\parbox[t]{2mm}{\multirow{3}{*}{\rotatebox[origin=c]{90}{{\small AliNet}}}} 
				& NNC & .676 & .419 & .517 & .738 & .558 & .634 & .597 & .482 & .534 & .761 & .120 & .207 & .466 & .365 & .409 & .545 & .162 & .250 \\
				& MR & .752 & .538 & .627 & .828 & .505 & .627 & .779 & .580 & .665 & .854 & .543 & .664 & .552 & .570 & .561 & .686 & .549 & .609 \\
				& BR & .762 & .556 & .643 & .829 & .515 & .635 & .783 & .591 & .673 & .846 & .546 & .663 & \underline{.547} & .556 & .552 & .674 & .556 & .609 \\
				\midrule			
				\parbox[t]{2mm}{\multirow{3}{*}{\rotatebox[origin=c]{90}{{\scriptsize MTransE}}}} 
				& NNC & .604 & .485 & .538 & .719 & .511 & .598 & .622 & .491 & .549 & .686 & .506 & .583 & .459 & .447 & .453 & .557 & .543 & .550 \\
				& MR & \underline{.781} & .702 & .740 & \underline{.866} & .675 & .759 & .799 & .708 & .751 & .864 & .653 & .744 & .482 & .575 & .524 & .639 & .613 & .625 \\
				& BR & \textbf{.811} & \underline{.728} & \underline{.767} & \textbf{.892} & \underline{.700} & \underline{.785} & \textbf{.816} & \underline{.733} & \underline{.772} & \textbf{.888} & \underline{.731} & \underline{.801} & .539 & \underline{.686} & \underline{.604} & \underline{.692} & \underline{.735} & \underline{.713}\\
				\midrule
				\multicolumn{2}{c}{Ours} & .763 & \textbf{.925} & \textbf{.836} & .844 & \textbf{.909} & \textbf{.875} & \underline{.807} & \textbf{.836} & \textbf{.821} & \underline{.880} & \textbf{.809} & \textbf{.843} & \textbf{.615} & \textbf{.772} & \textbf{.685} & \textbf{.732} & \textbf{.749} & \textbf{.740}\\
				\bottomrule
	\end{tabular}}}
	\caption{Dangling detection results on DBP2.0 in the consolidated setting.}
	\label{tab:our_dangling_ent_alignment}
\end{table*}

\subsection{Experiments Aware of Dangling Entities}\label{sec:aware}
We provide a comparison of dangling detection and entity alignment with baselines aware of dangling.

\textbf{Dangling Entities Detection Performance.}
We test our method's dangling detection performance compared with baselines aware of dangling entities. The results on DBP2.0 in the consolidated setting are reported in Tab.~\ref{tab:our_dangling_ent_alignment}. Note that the comparison is unfair as we don't use $30\%$ of the labeled dangling entities as the baselines. Nevertheless, our approach maintains SOTA performance across all six datasets, excelling in almost every metric except for a slightly inferior precision.

\textbf{Dangling-Entities-Aware Baselines Comparison.} Tab.~\ref{tab:alignment_ent_alignment} reports the entity alignment performance comparison in the consolidated setting on DBP2.0. The precision, recall, and F1 scores are computed according to Eq.~\eqref{eq:precAlign},~\eqref{eq:recAlign},~\eqref{eq:f1Align} in \textbf{Metric} part of appendix~\ref{app_7}, respectively. We test the entity alignment performance of our method in comparison with baselines that are aware of dangling entities. Our method still maintains almost state-of-the-art performance, but there is still a slightly inferior precision problem. It makes us wonder about the reasons behind it.

\textbf{How does our method work?}
\label{sec:visual}
To understand why our method works and its precision slightly suffers, we visualized all entity embeddings of GA16K in Fig.~\ref{fig:visual}. As shown above, matchable entities are denoted as red and green in source and target KG respectively, while dangling as blue. The distribution of three types of entity in Fig.~\ref{fig:visual}(a) suggests our method maps all nodes into a unified embedding space where matchable entities exhibit considerable overlap and are appropriately aligned (shown in Fig.~\ref{fig:visual}(b)). Fig.~\ref{fig:visual}(c)(d) depicts that a part of the dangling entities is intertwined with the matchable ones, suggesting that this part resides at the decision boundary and easily leads to false positives which explain the lesser precision of our method.

\begin{table*}[!t]
	\centering
	\renewcommand\arraystretch{0.8}
	\resizebox{.99\textwidth}{!}{\small
		{
			\small
			\setlength{\tabcolsep}{3pt}
			\begin{tabular}{llcccccccccccccccccc}
				\toprule
				\multicolumn{2}{c}{\multirow{2}{*}{Methods}} &
				\multicolumn{3}{c}{ZH-EN} & \multicolumn{3}{c}{EN-ZH} & \multicolumn{3}{c}{JA-EN} & \multicolumn{3}{c}{EN-JA} &  \multicolumn{3}{c}{FR-EN} & \multicolumn{3}{c}{EN-FR}\\
				\cmidrule(lr){3-5} \cmidrule(lr){6-8} \cmidrule(lr){9-11} \cmidrule(lr){12-14} \cmidrule(lr){15-17} \cmidrule(lr){18-20}
				&& Prec. & Rec. & F1 & Prec. & Rec. & F1 & Prec. & Rec. & F1 & Prec. & Rec. & F1 & Prec. & Rec. & F1 & Prec. & Rec. & F1 \\ 
				\midrule
				\parbox[t]{2mm}{\multirow{3}{*}{\rotatebox[origin=c]{90}{{\small AliNet}}}} 
				& NNC & .121 & .193 & .149 & .085 & .138 & .105 & .113 & .146 & .127 & .067 & .208 & .101 & .126 & .148 & .136 & .086 & .161 & .112 \\
				& MR & .207 & .299 & .245 & .159 & .320 & .213 & .231 & .321 & .269 & .178 & .340 & .234 & .195 & .190 & .193 & .160 & .200 & .178 \\
				& BR & .203 & .286 & .238 & .155 & .308 & .207 & .223 & .306 & .258 & .170 & .321 & .222 & .183 & .181 & .182 & .164 & .200 & .180 \\
				\midrule
				\parbox[t]{2mm}{\multirow{3}{*}{\rotatebox[origin=c]{90}{{\scriptsize MTransE}}}} 
				& NNC & .164 & .215 & .186 & .118 & .207 & .150 & .180 & .238 & .205 & .101 & .167 & .125 & .185 & .189 & .187 & .135 & .140 & .138 \\
				& MR & \underline{.302} & .349 & .324 & \underline{.231} & \.362 & .282 & .313 & \underline{.367} & \underline{.338} & .227 & \underline{.366} & .280 & \underline{.260} & \underline{.220} & \underline{.238} & \underline{.213} & \underline{.224} & .218 \\
				& BR & \textbf{.312} & \underline{.362} & \underline{.335} & \textbf{.241} & \underline{.376} & \underline{.294} & \underline{.314} & .363 & .336 & \textbf{.251} & .358 & \underline{.295} & \textbf{.265} & .208 & .233 & \textbf{.231} & .213 & \underline{.222} \\
				\midrule
				\multicolumn{2}{c}{Ours} & .279 & \textbf{.447} & \textbf{.344} & .219 & \textbf{.489} & \textbf{.303} & \textbf{.324} & \textbf{.409} & \textbf{.362} & \underline{.234} & \textbf{.460} & \textbf{.310} & .234 & \textbf{.320} & \textbf{.271} & .192 & \textbf{.363} & \textbf{.251}\\		
				\bottomrule
			\end{tabular}
		}
	}
	\caption{Entity alignment results on DBP2.0 in the consolidated setting.}
	\label{tab:alignment_ent_alignment}
\end{table*}

\subsection{Ablation Studies and Varying Anchor Nodes} \label{sec:anchor}
We conduct ablation studies to show each module's impact, and similarly pre-aligned entities' impact.

\textbf{Ablation Studies}
The impact of adaptive dangling indicator and relation projection attention in our method are investigated. We denote the counterpart removing $r_{e_i}$ as $w/o\,\, r_{e_i}$, and replacing $\bm{h}_{r_k}^{\to e_j}$ with the original $\bm{h}_{r_k}$ as $w/o\,\, \bm{h}_{r_k}^{\to e_j}$. Fig.~\ref{fig:ablation} gives the ablation study results on DBP2.0, where `Ours' represents an all-inclusive model. We observe that the $\bm{h}_{r_k}^{\to e_j}$ has a more substantial impact than $r_{e_i}$ to the alignment performance. As to why the $r_{e_i}$ has a minor impact on the alignment, we consider it may be attributed to the lower degrees of dangling entities on DBP2.0. The degrees of dangling entities are generally lower than that of matchable ones, indicating that the dangling is more isolated in the graph and thus has less impact on matchable nodes in the neighborhood aggregation.

\begin{figure*}[htbp]
	\centering
	\includegraphics[height = 0.13\textheight, width = 1\textwidth]{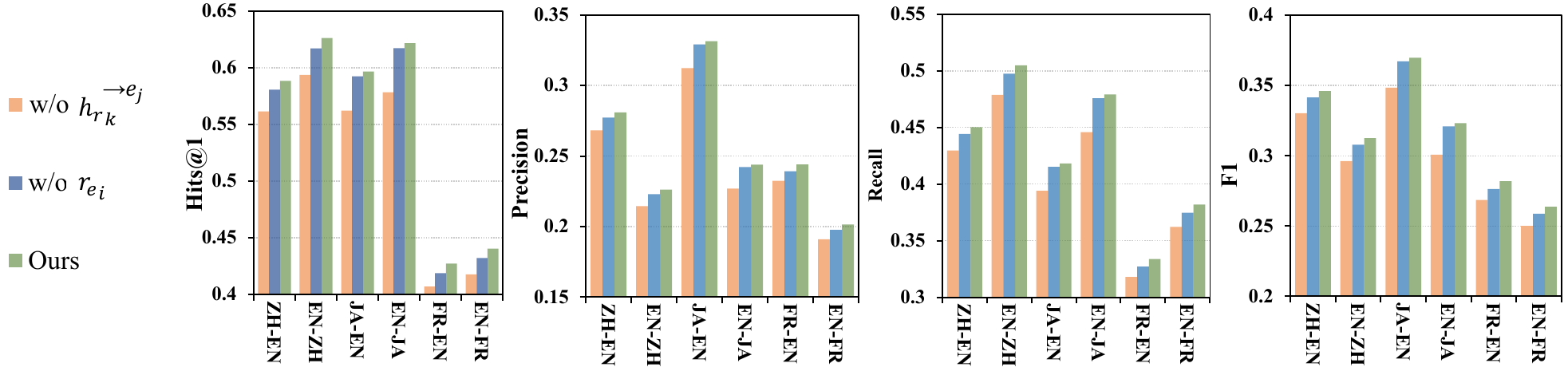}
	\caption{The ablation study of entity alignment performance in the consolidated setting on DBP2.0.}
	\label{fig:ablation}
\end{figure*}

\begin{figure*}[htbp]
	\centering
	\includegraphics[height = 0.12\textheight, width = 1\textwidth]{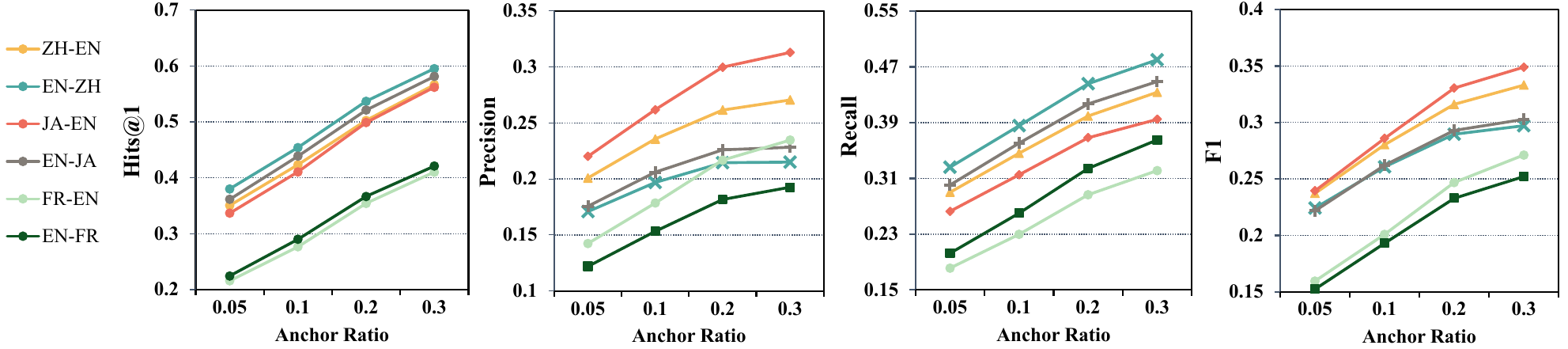}
	\caption{The entity alignment performance on varying pre-aligned anchor nodes ratios on DBP2.0.}
	\label{fig:ratios}
\end{figure*}

\textbf{Varying Anchor Nodes.}
Pre-aligned entities may be far scarce in reality. The sensitivity of our method to the proportion variation of anchor nodes is investigated. As the proportion increases, the alignment performance enhances as provided in Fig.~\ref{fig:ratios}.

Notably, even with an anchor ratio as low as 5\%, our alignment accuracy still well exceeds 30\% on most datasets except for FR-EN and EN-FR. Cause they contain twice as many entities and triples as ZH-EN and JA-EN, which introduces intricate dependencies among entities and thus greater challenges in alignment. Moreover, a larger graph may require a higher dimension of representations to learn, but the embedding dimension is restricted to merely 96 due to the out-of-memory problem.

\section{Conclusion}
We found that previous EA methods suffer from great performance decline if dangling entities are considered. Our goal is to address the EA problem with unlabeled dangling entities. A novel framework Lambda for detecting dangling entities and then pairing alignment is proposed. The core idea is to perform selective aggregation with spectral contrastive learning and to adopt theoretically guaranteed PU learning to relieve the dependence on the labeled dangling entities. Experimental results on multiple representative datasets demonstrate the effectiveness of our proposed approach. This work also has important implications for real-world applications, such as EA of different scales, KG plagiarism detection, etc.

\newpage
\begin{ack}
The research was supported in part by NSF China (No. 61960206002, 62272306, 62032020, 62136006).

The authors would like to thank the reviewers for their constructive comments and appreciate the Student Innovation Center of SJTU for providing GPUs. Hang Yin personally thanks Chenyu Liu, Yuting Feng, Jingyuan Zhou, and Qingyang Liu for feedbacks on early versions of this paper. Hang Yin would also like to thank Professor Yuan Luo for his inspiration in the information theory course (CS7317) of Shanghai Jiao Tong University.
\end{ack}

\bibliographystyle{plain}

\newpage
\appendix

\begin{appendix}
	\thispagestyle{plain}
	\begin{center}
		{\Large \bf Appendix}
	\end{center}
\end{appendix}


\section{Notation}\label{app_0}
\subsection{Definitions}
\begin{definition}[Knowledge graph] Knowledge graph (KG) is a directed graph $G=(E,R,T)$ comprising three distinct sets: entities $E$, relations $R$, and triples $T \subseteq E \times R \times E$. KG is stored in the form of triples $<$entity, relation, entity$>$, with entities denoted by nodes and the relation between entities defined by edges. 
\end{definition}

\begin{definition}[Entity alignment]
	Given source KG and target KG, corresponding to $G_s=(E_s, R_s, T_s)$ and $G_t=(E_t, R_t, T_t)$ respectively, and $A=\left\{(u,v)|u \in E_s, \newline v \in E_t, u \equiv v \right\}$ a set of pre-aligned anchor node pairs, where $\equiv$ indicates equivalence, the goal of entity alignment is to identify additional pairs of potentially equivalent entities using information from $G_s$, $G_t$, and $A$. This task typically assumes a one-to-one correspondence between $E_s$ and $E_t$.
\end{definition}

\begin{definition}[Entity alignment with dangling cases]
	Let entities in the source and target graphs be composed of two types of nodes: $E_s=D_s \cup M_s, E_t=D_t \cup M_t$, where $D_s, D_t$ denote dangling sets that contain entities that have no counterparts, and $M_s, M_t$ are matchable sets. A set of pre-aligned anchor node pairs are $S=\left\{(u,v)|u \in M_s, v \in M_t, u \equiv v\right\}$. The task seeks to discover the remaining aligned entities given $G_s$, $G_t$, and $S$. 
\end{definition}

\subsection{Transductive Learning:}
Transductive learning models are trained from observed, specific (training) cases to specific (test) cases, employing both training and test information except for test labels. In contrast, the inductive learning model is reasoning from observed training cases to general rules, which are then applied to the test cases. Let's get down to the EA task. As KG structure is accessible through given triples, which can accurately describe the connections between entities. When we attempt to figure out entity alignment tasks, the KG structure information of the whole source and target KGs is what we could exploit, covering potentially (test) equivalent entities' relative positions. That's why we confine the problem field to transductive learning.

\subsection{Graph Convolutional Networks}
For graph convolutional networks (GCN) \cite{kipf2016semi}, the embedding $\bm{h}^{l+1}_{e_i}$ of node ${e_i}$ at the $l+1$-th layer is updated iteratively by aggregating node features of the neighboring nodes $\mathcal{N}_{e_i}$ from the prior layer:
\begin{equation}
	\bm{h}^{l+1}_{e_i} = \sigma\left( \sum_{e_j\in \mathcal{N}_{e_i} \cup \{e_i\}} \alpha_{i,j} W^{l+1} \bm{h}^{l}_{e_j}\right),
\end{equation}
where each embedding $\bm{h}^{l}_{e_i}$ represents the $d$-dimensional embedding vector of $e_i$, $\alpha_{i,j}$ denotes the weight coefficient between $e_i$ and $e_j$, $W^{l+1}$ being the transformation matrix of the $(l+1)$-th GNN layer, and $\sigma$ being the activation function.

\section{Proof for Lemma~\ref{spec_hard}}\label{app_5}
\begin{proof}
	
	The $\mathcal{L}_{\mathrm{info}}$ also has a good effect on mining high-quality negative samples, which we show has an equivalent effect to \textit{truncated uniform negative sampling (TUNS)} in \cite{sun2018bootstrapping}. TUNS points out that negative samples obtained by uniform random sampling are highly redundant since only high-quality negative samples improve the model. Thus TUNS chooses the K-nearest neighbors of the $e_i$ as the negative samples, which are most challenging to distinguish. In the special case of $K = 1$, the loss can be written as $\mathcal{L}_{\mathrm{TUNS}} = \sum_{e_i\in \mathcal{X}_p} \max_{j}(H(e_i, e^i_{+}, e^i_{j})).$ If we approximate the $\max$ function by the LogSumExp, the contrastive loss function turns to
	\begin{equation}
		\mathcal{L}_{\mathrm{TUNS}}  \approx \sum_{e_i\in \mathcal{X}_p} \frac{1}{\lambda} \log\left(\sum_j^{N}\exp(\lambda~ H(e_i, e^i_{+}, e^i_{j}))\right),
	\end{equation}
	minimizing which is equivalent to minimizing Eq.~\eqref{L1}. Hence our contrastive learning loss is actually a special form of TUNS. For randomly sampled negative samples, $\mathcal{L}_{\mathrm{info}}$ can play a role in preferentially optimizing high-quality negative samples.\\
\end{proof}

\section{Proof for Theorem~\ref{theorem_unbiased}}\label{app_2}
\begin{proof}
	
	The risk of $g$ is $R(g)=\mathbb{E}_{(X,Y)\sim p(x,y)}[\ell(g(X),Y)]=\pi_p R_\mathrm{p}^+(g)+\pi_n R_\mathrm{n}^-(g)$ in positive negative learning problems. In positive-unlabeled learning where $\mathcal{X}_\mathrm{n}$ is unavailable, we can only approximate $R(g)$ by positive samples and unlabeled samples. We represent the unlabeled distribution as $p_\mathrm{u}(x) = \pi^\mathrm{u}_\mathrm{n} p_\mathrm{n}(x) + \pi^\mathrm{u}_\mathrm{p} p_\mathrm{p}(x)$, so that the negative distribution can be written as $\pi_\mathrm{n} p_\mathrm{n}(x) =\frac{\pi_\mathrm{n}}{\pi^u_\mathrm{n}} \cdot \left[p_\mathrm{u}(x) - \pi^u_\mathrm{p} p_\mathrm{p}(x)\right]$. Provided $R_\mathrm{p}^-(g)=\mathbb{E}_{X\sim p_\mathrm{p}(x)}[\ell(g(X),-1)]$ and $R_\mathrm{u}^-(g)=\mathbb{E}_{X\sim p_\mathrm{u}(x)}[\ell(g(X),-1)]$, we obtain that 
	\begin{equation}
		\pi_\mathrm{n} R_\mathrm{n}^-(g) = \frac{\pi_\mathrm{n}}{\pi^\mathrm{u}_\mathrm{n}} \cdot \left[ R_\mathrm{u}^-(g)-\pi^\mathrm{u}_\mathrm{p} R_\mathrm{p}^-(g) \right],
	\end{equation}
	and 
	\begin{equation}
		\widehat{R}_{\mathrm{pu}}(g)= \pi_{\mathrm{p}} \widehat{R}_{\mathrm{p}}^+(g)+ \frac{\pi_\mathrm{n}}{\pi^\mathrm{u}_\mathrm{n}} \cdot \left[ \widehat{R}_{\mathrm{u}}^-(g)-\pi^\mathrm{u}_{\mathrm{p}}\widehat{R}_{\mathrm{p}}^-(g) \right].
		\label{proof_Rpu}
	\end{equation}
	Specifically, $\pi_\mathrm{n} = 1 - \pi_\mathrm{p}$, $\pi^\mathrm{u}_\mathrm{n}=1-\pi^\mathrm{u}_\mathrm{p}$ could be derived as \textit{class-prior probability given $\pi_\mathrm{p}$ and $\pi^\mathrm{u}_\mathrm{p}$}. In particular, the ratio of labeled positive samples could be precisely figured out as $\pi^{tr}_{\mathrm{p}}$ in transductive learning, given which $\pi^\mathrm{u}_\mathrm{p} = \frac{\pi_\mathrm{p}-\pi^{tr}_{\mathrm{p}}}{1-\pi^{tr}_\mathrm{p}}$, $\pi^\mathrm{u}_\mathrm{n}=1-\pi^\mathrm{u}_\mathrm{p}$ could be derived as \textit{class-prior probability}.
	
\end{proof}


\section{Proof for Theorem~\ref{deviation_bound}}\label{app_3}
\begin{proof}
	
	$\mathfrak{R}_{n,q}$ is defined as the \textit{Rademacher complexity} of the class of classifiers $\mathcal{G}$ for the sampling of size $n$ from distribution $q(x)$. From \cite{niu2016theoretical} we have that with probability at least $1-\delta/2$, the uniform deviation bounds below hold separately:
	\begin{equation}
		\begin{aligned}
			\sup_{g\in\mathcal{G}}|\widehat{R}_+(g)-R_+(g)|&\le2L_\ell\mathfrak{R}_{n_+,p_\mathrm{p}}(\mathcal{G})+\sqrt{\frac{\ln(4/\delta)}{2n_+}}\\
			&\triangleq M_+>0,\\
			\sup_{g\in\mathcal{G}}|\widehat{R}_{\mathrm{u},-}(g)-R_{\mathrm{u},-}(g)|&\le2L_\ell\mathfrak{R}_{n_\mathrm{u},p_\mathrm{u}}(\mathcal{G})+\sqrt{\frac{\ln(4/\delta)}{2n_\mathrm{u}}}\\
			&\triangleq M_->0,
		\end{aligned}
		\notag
	\end{equation}
	where $L_\ell$ is the \textit{Lipschitz constant} of loss $\ell$ in its first parameter and $n_+$ is the number of positive samples while $n_u$ is that of unlabeled. Following the \textit{symmetric condition assumption} in \cite{niu2016theoretical}, it is obviously holds that:
	\begin{equation}
		\begin{aligned}
			\mathrm{Var}(\widehat{R^{'}}_{\mathrm{pu}}(g)) =\, & 2\pi_\mathrm{p} M_{+} + M_{-}, \textrm{while}\\  
			\mathrm{Var}(\widehat{R}_{\mathrm{pu}}(g)) =\, & (\pi_\mathrm{p}+\frac{\pi_\mathrm{n} \cdot \pi_\mathrm{p}^\mathrm{u}}{\pi_\mathrm{n}^\mathrm{u}}) M_{+} +\frac{\pi_\mathrm{n}}{\pi_\mathrm{n}^\mathrm{u}} M_{-}
		\end{aligned}
		\notag
	\end{equation}
	holds in our setting. Then it is evident that $\frac{\pi_\mathrm{n}}{\pi_\mathrm{p}} < \frac{\pi_\mathrm{n}^u}{\pi_\mathrm{p}^u}$, i.e., $\frac{\pi_\mathrm{n} \cdot \pi_\mathrm{p}^\mathrm{u}}{\pi_\mathrm{n}^\mathrm{u}} < \pi_\mathrm{p}$ and $\frac{\pi_\mathrm{n}}{\pi_\mathrm{n}^u}<1$. 
	
	Consequently, comparing the coefficients of $M_+$ and $M_-$ leads to the conclusion that $\widehat{R}_{\mathrm{pu}}(g)$ could possess tighter uniform deviation bound than that of \textit{Non-negative Risk Estimator} \cite{kiryo2017positive}.
	
\end{proof}

\section{Convergence Proof}\label{app_1}
The expectation–maximization (EM) algorithm is an iterative approach to maximize the likelihood $p(\mathbf{y}|\mathbf{X};\theta)$ of target variables $y$ over input variables $\mathbf{X}$ and parameters $\theta$. The EM algorithm works iteratively, and each iteration consists of an expectation (E) step and a maximization (M) step. The E step computes the expectation of the log-likelihood concerning the conditional distribution of the latent variable $\mathbf{z}$ given the current parameters $\theta^{(t)}$ at the $t$-th iteration:

\begin{equation}
	\begin{aligned}
		&Q(\theta\mid\theta^{(t)})=\mathbb{E}_{\mathbf{z}\sim p(\mathbf{z}\mid\mathbf{X},\mathbf{y},\theta^{(t)})}[\log p(\mathbf{y},\mathbf{z}\mid\mathbf{X},\theta)].\\
	\end{aligned}
	\label{Q_func}
\end{equation}

Then, the M step finds a set of parameters that maximizes the computed expectation:
\begin{equation}
	\begin{aligned}
		&\theta^{(t+1)}=\arg\max_{\theta}Q(\theta\mid\theta^{(t)}).\\
	\end{aligned}
\end{equation}
Due to the maximization step, we get the following inequality naturally:
\begin{equation}
	\begin{aligned}
		&Q(\theta^{(t+1)} \mid \theta^{(t)})-Q(\theta^{(t)} \mid \theta^{(t)})\ge 0.
	\end{aligned}
\end{equation}

\begin{lemma}[Convergence of EM algorithm \cite{murphy2012machine}]
	It is guaranteed that the EM algorithm always improves $\log p(\mathbf{y}\mid\mathbf{X},\theta)$ by increasing the value of $Q(\theta\mid\theta^{(t)})$. Since $\log p(\mathbf{y}\mid\mathbf{X},\theta)$ is monotonically bounded, the EM must converge.
	\label{theorem_converge}
\end{lemma}

\begin{proof}
	The following equation holds for any z:
	\begin{equation}
		\begin{aligned}
			\log p(\mathbf{y}\mid\mathbf{X},\theta)=\log p(\mathbf{y},\mathbf{z}\mid\mathbf{X},\theta)-\log p(\mathbf{z}\mid\mathbf{X},\mathbf{y},\theta).
			\notag
		\end{aligned}
		\label{log_p}
	\end{equation}
	
	We take the expectation over $p(\mathbf{z}\mid\mathbf{X},\mathbf{y},\theta^{(t)})$ for both sides as follows:
	\begin{equation}
		\begin{aligned}
			&\quad\mathbb{E}_{p(\mathbf{z}|\mathbf{X},\mathbf{y},\theta^{(t)})}[\log p(\mathbf{y}\mid\mathbf{X},\theta)] \\
			&=\mathbb{E}_{p(\mathbf{z}|\mathbf{X},\mathbf{y},\theta^{(t)})}[\log p(\mathbf{y},\mathbf{z}\mid\mathbf{X},\theta)]-\mathbb{E}_{p(\mathbf{z}|\mathbf{X},\mathbf{y},\theta^{(t)})}[\log p(\mathbf{z}\mid\mathbf{X},\mathbf{y},\theta)] \\
			&=Q(\theta\mid\theta^{(t)})+H(p(\mathbf{z}\mid\mathbf{X},\mathbf{y},\theta)\mid p(\mathbf{z}\mid\mathbf{X},\mathbf{y},\theta^{(t)}))\\
			&\triangleq Q(\theta\mid\theta^{(t)})+H(p_\theta\mid p_{\theta^{(t)}})
			\notag
		\end{aligned}
		\label{E_theta_t}
	\end{equation}
	where $H$ stands for the entropy. If we substitute $\theta^{(t)}$ for $\theta$, we get the following:
	\begin{equation}
		\log p(\mathbf{y}\mid\mathbf{X},\theta^{(t)})=Q(\theta^{(t)}\mid\theta^{(t)})+H(p_{\theta^{(t)}}\mid p_{\theta^{(t)}}).
		\label{theta_t}
	\end{equation}
	
	Gibbs’ inequality states that $H(q\mid p)-H(p\mid p)\geq0$ always holds for any distribution $p$ and $q$. Hence we have:
	\begin{equation}
		\begin{aligned}
			&\quad\log p(\mathbf{y}\mid\mathbf{X},\theta^{(t+1)})-\log p(\mathbf{y}\mid\mathbf{X},\theta^{(t)}) \\
			&= Q(\theta^{(t+1)}\mid\theta^{(t)})+H(p_{\theta^{(t+1)}}\mid p_{\theta^{(t)}})-Q(\theta^{(t)}\mid\theta^{(t)})-H(p_{\theta^{(t)}}\mid p_{\theta^{(t)}}) \\
			&=Q(\theta^{(t+1)}\mid\theta^{(t)})-Q(\theta^{(t)}\mid\theta^{(t)})+H(p_{\theta^{(t+1)}}\mid p_{\theta^{(t)}})-H(p_{\theta^{(t)}}\mid p_{\theta^{(t)}}) \\
			&\geq Q(\theta^{(t+1)}\mid\theta^{(t)})-Q(\theta^{(t)}\mid\theta^{(t)}) = Q(\theta^{(t+1)}\mid\theta^{(t)})-Q(\theta^{(t)}) \ge 0.\notag
		\end{aligned}
		\label{final_conclusion}
	\end{equation}
\end{proof}

Now we provide the proof for Theorem~\ref{theorem_Q_Rpu}.
\begin{proof}

	According to Lemma~\ref{theorem_converge}, we have that the EM algorithm converges. Next, we model our problem and give the approximate equivalence between our algorithm and the EM algorithm to prove our algorithm converges.
	
	The latent variables $\mathbf{z}$ represent the true label distribution of unlabeled samples, where $\hat{y}_{i}(u)=f(\mathbf{X},i;\theta)$ is the probability of node $i$ being labeled as $u \in \{ +1 , -1 \}$ by the current classifier $f$. Given $\pi_\mathrm{p}^\mathrm{u}$, the label distribution of all unlabeled samples is:
	\begin{equation}
		\begin{aligned}
			p(z_i)=
			\begin{cases}
				\hat{\pi}_\mathrm{p}^\mathrm{u}&\text{if\quad}z_i=+1,\\
				1-\hat{\pi}_\mathrm{p}^\mathrm{u}&\text{if\quad}z_i=-1.
			\end{cases}
		\end{aligned}
		\label{z}
	\end{equation}
	
	First, the conditional distribution $p(\mathbf{z}\mid\mathbf{X},\mathbf{y},\theta^{(t)})$ of latent variables given the current parameters $\theta^{(t)}$ is approximated by:
	\begin{equation}
		\begin{aligned}
			&p(\mathbf{z}\mid\mathbf{X},\mathbf{y},\theta^{(t)})\approx\prod_{i\in\mathcal{U}}p(z_i).\\
		\end{aligned}
		\label{z_margin}
	\end{equation}
	
	Second, the joint distribution $p(\mathbf{y},\mathbf{z}\mid\mathbf{X},\theta)$ of labeled and unlabeled nodes with new parameters $\theta$ is approximated by the classifier $f$, which is also considered as a marginalization function that gives the label distribution of each node based on all given information:
	\begin{equation}
		\begin{aligned}
			&p(\mathbf{y},\mathbf{z}\mid\mathbf{X},\theta)\approx\prod_{i\in\mathcal{P}}\hat{y}_{i}(+1)\prod_{j\in\mathcal{U}}\hat{y}_{j}(z_{j})
		\end{aligned}
		\label{y_z_margin}
	\end{equation}
	
	We propose to use the average log-likelihood differences to measure the classification preference of the model. It can be transformed into the following form by $\Delta_{\mathcal{U}}$ and $\Delta_{\mathcal{P}}$, representing that on domain $\mathcal{U}$ and $\mathcal{P}$: 
	\begin{equation}
		\begin{aligned}
			&\Delta_{\mathcal{U}} = \frac{1}{|\mathcal{U}|}\sum_{j\in\mathcal{U}}\log \hat{y}_j(+1) - \log \hat{y}_j(-1)\\
			&\Delta_{\mathcal{P}} = \frac{1}{|\mathcal{P}|}\sum_{i\in\mathcal{P}}\log\hat{y}_i(+1) - \log\hat{y}_i(-1)\notag
		\end{aligned}
	\end{equation}
	
	If the mean value is positive, it means that the logarithmic probability of positive classes is higher than that of negative classes in the given domain, and vice versa. When we ignore the preference of the classification model, this actually describes the category feature bias in certain domains.
	
	Let's rethink the meaning of \textbf{\textit{preference condition}}. If the model has a similar classification preference in domain $\mathcal{U}$ and $\mathcal{P}$, we can express as $\Delta_{\mathcal{U}} \approx \Delta_{\mathcal{P}}$. This condition can be arranged as the mathematical expression of the condition given by Theorem~\ref{theorem_Q_Rpu} through the properties of the $\log$ function.
	
	\begin{equation}
		\begin{aligned}
			&\Delta_{\mathcal{U}} \approx \Delta_{\mathcal{P}} \equiv \sum_{j\in\mathcal{U}}\frac{1}{|\mathcal{U}|}\log\frac{\hat{y}_j(+1)}{\hat{y}_j(-1)} \approx \sum_{i\in\mathcal{P}}\frac{1}{|\mathcal{P}|}\log\frac{\hat{y}_i(+1)}{\hat{y}_i(-1)}\notag
		\end{aligned}
	\end{equation}
	
	We derive $-\mathcal{L}_{\mathrm{pu}}$ from Eq.(~\ref{pnrisk}), since the goal of training is to minimize the objective function:
	\begin{equation}
		\begin{aligned}
			&\quad\frac{1}{N}\mathbb{E}_{\mathbf{z}\sim p(\mathbf{z}|\mathbf{X},\mathbf{y},\theta^{(t)})}[\log p(\mathbf{y},\mathbf{z}\mid\mathbf{X},\theta)] \\
			&=\frac{1}{N}\sum_{\mathbf{z}}p(\mathbf{z}\mid\mathbf{X},\mathbf{y},\theta^{(t)})\log p(\mathbf{y},\mathbf{z}\mid\mathbf{X},\theta) \\
			&\approx\frac{1}{N}\sum_{\mathbf{z}}p(\mathbf{z}\mid\mathbf{X},\mathbf{y},\theta^{(t)})(\sum_{i\in\mathcal{P}}\log\hat{y}_{i}(+1)+\sum_{j\in\mathcal{U}} 
			\log\hat{y}_{j}(z_{j})) \\
			&=\frac{1}{N}\sum_{i\in\mathcal{P}}\log\hat{y}_{i}(+1)+\frac{1}{N}\sum_{j\in\mathcal{U}} 
			\sum_{z_{j}\in\pm1} p(z_j) \log\hat{y}_{j}(z_{j}) \\
			&=\frac{|\mathcal{P}|}{N}\cdot\frac{1}{|\mathcal{P}|}\sum_{i\in\mathcal{P}}\log\hat{y}_{i}(+1)+\frac{|\mathcal{U}|}{N}\cdot\frac{1}
			{|\mathcal{U}|}\sum_{j\in\mathcal{U}}\left(\hat{\pi}_\mathrm{p}^\mathrm{u}\log\hat{y}_{j}(+1)+(1-\hat{\pi}_\mathrm{p}^\mathrm{u})\log\hat{y}_{j}(-1)\right)\notag\\
			&=\frac{|\mathcal{P}|}{N}\cdot\frac{1}{|\mathcal{P}|}\sum_{i\in\mathcal{P}}\log\hat{y}_{i}(+1)+\frac{|\mathcal{U}|}{N}\cdot\frac{1}{|\mathcal{U}|}\sum_{j\in\mathcal{U}}\log\hat{y}_{j}(-1) + \frac{\hat{\pi}_\mathrm{p}^\mathrm{u}|\mathcal{U}|}{N}\Delta_{\mathcal{U}}\\
		\end{aligned}
	\end{equation}
	
	We replaced $\Delta_{\mathcal{U}}$ with $\Delta_{\mathcal{P}}$, and this process is completed by the above \textbf{\textit{preference condition}}.
	\begin{equation}
		\begin{aligned}
			&\approx\frac{|\mathcal{P}|}{N}\cdot\frac{1}{|\mathcal{P}|}\sum_{i\in\mathcal{P}}\log\hat{y}_{i}(+1)+\frac{|\mathcal{U}|}{N}\cdot\frac{1}
			{|\mathcal{U}|}\sum_{j\in\mathcal{U}}\log\hat{y}_{j}(-1) + \frac{|\mathcal{U}|\hat{\pi}_\mathrm{p}^\mathrm{u}}{N}\cdot\frac{1}
			{|\mathcal{P}|}\sum_{i\in\mathcal{P}}\log\hat{y}_{i}(+1)-\log\hat{y}_{i}(-1)\\
			&=-\frac{|\mathcal{P}|+|\mathcal{U}| \cdot \hat{\pi}_{\mathrm{p}}^{\mathrm{u}}}{|\mathcal{P}|+|\mathcal{U}|}\sum_{i\in\mathcal{P}}\frac{1}{|\mathcal{P}|}\log\hat{y}_{i}(+1)-\frac{\frac{|\mathcal{N}|}{|\mathcal{P}|+|\mathcal{U}|}}{\frac{|\mathcal{N}|}{|\mathcal{U}|}}  \left( \sum_{j\in\mathcal{U}}\frac{1}{|\mathcal{U}|}\log\hat{y}_{j}(-1)\right.\left.-\hat{\pi}_\mathrm{p}^\mathrm{u}\sum_{i\in\mathcal{P}}\frac{1}{|\mathcal{P}|}\log\hat{y}_{i}(-1) \right)
			\notag
		\end{aligned}
	\end{equation}
	
	Denote $\pi_{\mathrm{p}} = \frac{|\mathcal{P}|+|\mathcal{U}| \cdot \hat{\pi}_{\mathrm{p}}^{\mathrm{u}}}{|\mathcal{P}|+|\mathcal{U}|}, \pi_{\mathrm{n}} = \frac{|\mathcal{N}|}{|\mathcal{P}|+|\mathcal{U}|}$ and $\pi_{\mathrm{n}}^{\mathrm{u}} = \frac{|\mathcal{N}|}{|\mathcal{U}|}$. The above formula is equivalent to:
	\begin{equation}
		\begin{aligned}
			&\arg \max_\theta \frac{1}{N}\mathbb{E}_{\mathbf{z}\sim p(\mathbf{z}|\mathbf{X},\mathbf{y},\theta^{(t)})}[\log p(\mathbf{y},\mathbf{z}\mid\mathbf{X},\theta)] \\
			&= \arg \max_\theta \underbrace{-\pi_{\mathrm{p}} \widehat{R}_{\mathrm{p}}^+(g)- \frac{\pi_\mathrm{n}}{\pi^\mathrm{u}_\mathrm{n}} \cdot \left[ \widehat{R}_{\mathrm{u}}^-(g)-\pi^\mathrm{u}_{\mathrm{p}}\widehat{R}_{\mathrm{p}}^-(g) \right]}_{-\widehat{R}_{\mathrm{pu}}(g)}\\
			&= \arg \min_\theta\widehat{R}_{\mathrm{pu}}(g) = \arg \max_\theta-\mathcal{L}_{\mathrm{pu}}
			\label{proof_Rpu_EM}
		\end{aligned}
	\end{equation}
\end{proof}

\section{Statistics of Experimental Dataset and Baselines}\label{app_7}
\textbf{Datasets.} The training/test sets for each dataset are generated using a fixed random seed. For entity alignment, 30\% of matchable entity pairs constitute the training set, while the remaining form the test set. For dangling entity detection, we did not utilize any labeled dangling entity data, in contrast to prior work which labels 30\% of the dangling entities and matchable pairs respectively for training \cite{sun2021knowing}. Hence our method imposes minimal restrictions on annotated data. All datasets are briefly introduced in the following and some statistics are provided in Tab.~\ref{tab:static_1}, ~\ref{tab:static_2}, ~\ref{tab:static_3}. 

In addition to the existing datasets, we also constructed DBP2.0-minus \& -plus as supplementary to DBP2.0, GA16K enabling comparison between Dangling-Entities-Unaware baselines, and GA-DBP15K for evaluation of iPULE.

\begin{table}[t]
	\begin{center}
		\renewcommand\arraystretch{0.8}
		\setlength{\tabcolsep}{10pt}
		\begin{tabular}{p{1.55cm}c|cccccccc}
			\toprule
			\multicolumn{2}{c|}{Datasets} & \# Entities & \# Rel.  & \# Triples & \# Dang &\# Align \\
			\toprule
			\multirow{2}{1.3cm}{$\rm{DBP2.0_{ZH-EN}}$} & Chinese & 84,996 & 3,706& 286,067 & 51,813&\multirow{2}{0.7cm}{33,183}\\
			& English & 118,996 & 3,402 & 586,868 & 85,813\\
			\multirow{2}{1.3cm}{$\rm{DBP2.0_{JA-EN}}$} & Japanese & 100,860 & 3,243 & 347,204 & 61,090 &\multirow{2}{0.7cm}{39,770}\\
			& English & 139,304 & 3,396  & 668,341  & 99,534\\
			\multirow{2}{1.3cm}{$\rm{DBP2.0_{FR-EN}}$} & French & 221,327 & 2,841 & 802,678 & 97,375 &\multirow{2}{0.8cm}{123,952}\\
			& English & 278,411 & 4,598 & 1,287,231 & 154,459 \\
			\hline
			\multirow{2}{1.3cm}{$\rm{DBP15K_{ZH-EN}}$} & Chinese & 19,388 & 1,701& 70,414 & 4,388 &\multirow{2}{0.7cm}{15,000}\\
			& English & 19,572 & 1,323 & 95,142 & 4,572\\
			\multirow{2}{1.3cm}{$\rm{DBP15K_{JA-EN}}$} & Japanese & 19,814 & 1,299 & 77,214 & 4,814 &\multirow{2}{0.7cm}{15,000}\\
			& English & 19,780 & 1,153  & 93,484  & 4,780\\
			\multirow{2}{1.3cm}{$\rm{DBP15K_{FR-EN}}$} & French & 19,661 & 903 & 105,998 & 4,661 &\multirow{2}{0.7cm}{15,000}\\
			& English & 19,993 & 1,208 & 115,722 & 4,993 \\
			\hline
			\multirow{2}{1.3cm}{$\rm{GA16K}$} & None & 6,208 & 8 & 68,534 & 0 &\multirow{2}{0.7cm}{6,208}\\
			& None & 16,363 & 12 & 151,662 & 10,155 \\
			\bottomrule
		\end{tabular}
\end{center}
\caption{Statistics of DBP2.0, DBP15K and GA16K.}
\label{tab:static_1}
\end{table}

\begin{table}[t]
\begin{center}
	\renewcommand\arraystretch{0.9}
	\setlength{\tabcolsep}{10pt}
	\begin{tabular}{p{2.4cm}c|cccccccc}
		\toprule
		\multicolumn{2}{c|}{Datasets} & \# Entities & \# Rel.  & \# Triples & \# Dang &\# Align \\
		\toprule
		\multirow{2}{1.3cm}{$\rm{DBP2.0_{ZH-EN}Plus}$} & Chinese & 69,386 & 3,455 & 241,588 & 36,302 &\multirow{2}{0.7cm}{33,084}\\
		& English & 94,026 & 3,131 & 470,284 & 60,942\\
		\multirow{2}{1.3cm}{$\rm{DBP2.0_{JA-EN}Plus}$} & Japanese & 82,192 & 3,011 & 291,406 & 42,588 &\multirow{2}{0.7cm}{39,604}\\
		& English & 110,362 & 3,054 & 532,988  & 70,758\\
		\hline
		\multirow{2}{1.3cm}{$\rm{DBP2.0_{ZH-EN}Minus}$} & Chinese & 72,252 & 3,351 & 200,400 & 54,594 &\multirow{2}{0.7cm}{17,658}\\
		& English & 107,853 & 3,140 & 421,597 & 90,195\\
		\multirow{2}{1.3cm}{$\rm{DBP2.0_{JA-EN}Minus}$} & Japanese & 86,241 & 3,014 & 236,546 & 64,841 &\multirow{2}{0.7cm}{21,400}\\
		& English & 126,558 & 3,166 & 485,133  & 105,158\\
		\bottomrule
	\end{tabular}
\end{center}
\caption{Statistics of DBP2.0-Plus and DBP2.0-Minus}
\label{tab:static_2}
\end{table}

\textbf{DBP15K}\footnote{https://paperswithcode.com/dataset/dbp15k} \cite{sun2017cross}: DBP15K consists of three cross-lingual subsets constructed from DBpedia: English-French($\rm{DBP_{FR-EN}}$), English-Chinese ($\rm{DBP_{ZH-EN}}$), English-Japanese($\rm{DBP_{JA-EN}}$). Each subset contains 15,000 pre-aligned entity pairs. This dataset includes a small proportion of dangling entity samples which is yet mostly ignored in previous entity alignment tasks.

\textbf{DBP2.0}\footnote{https://github.com/nju-websoft/OpenEA/tree/master/dbp2.0} \cite{sun2021knowing}: DBP2.0 is an entity alignment dataset with a considerable proportion of dangling entities, constructed from the multilingual Infobox Data of DBpedia \cite{auer2007dbpedia}. The dataset contains three pairs of crosslingual KGs, ZH-EN (Chinese to English), JA-EN (Japanese to English), and FR-EN (French to English). Since there are dangling nodes in both the source and target graphs, we separately test source-to-target and target-to-source alignment, consistent with the established approach. A representative feature of the dataset is that the matchable and dangling entities exhibit similar degree distributions which are hard to distinguish, displaying a real-world challenge in aligning knowledge graphs. Based on DBP2.0, we extend the following -minus \& -plus datasets for verification of iPULE on different positive proportions between 20\%-50\%.

\textbf{DBP2.0-plus}: In the construction of the plus dataset, our goal is to construct the dataset that has a higher $\pi_{p}$, and we realize this by reducing a few existing dangling entities on ZH-EN and JA-EN. We randomly delete dangling entities from both source and target KG equally and remove triples containing them. The constructed DBP2.0-plus are reindexed and thus obtain a higher $\pi_{\mathrm{p}}$ value than the original dataset.

\textbf{DBP2.0-minus}: In contrast, to lower the $\pi_{\mathrm{p}}$ value. Given the constraint of preventing new dangling entities that could introduce false information to the KG, we can only reduce the number of matchable entities. Given source and target KG, removing one entity from a pair makes the remaining entity dangling. We randomly delete matchable entities from one side of the pair on both source and target KG uniformly. The constructed DBP2.0-minus are reindexed and thus obtain a lower $\pi_{\mathrm{p}}$ value than the original dataset.


\textbf{GA16K}: This dataset constructed by us exclusively contains dangling nodes in the target graph, facilitating a comparison between our work and baselines that neglect dangling entities. GA16K is extracted from GAKG\footnote{https://github.com/davendw49/gakg} \cite{deng2021gakg}, a Geoscience Academic Knowledge Graph. We first order each type of entity in GAKG according to their degrees and select the entities with a large degree into the entity set. A total of 16,363(16K) separate entities and their relations were extracted to compose the target graph. Then we extract 6,208 entities from the target graph to comprise the source graph. Hence there are 6,208 ground-truth matchable pairs between the source and the target. The remaining 10,155 entities in the target graph are regarded as dangling entities.

\begin{table}[t]
\begin{center}
\renewcommand\arraystretch{0.5}
\setlength{\tabcolsep}{15pt}
\begin{tabular}{p{1cm}c|cccccccc}
	\toprule
	\multicolumn{2}{c|}{GA-DBP15K} & Entities & Dang & Align \\
	\toprule
	\multirow{2}{1.3cm}{$\rm{GA-EN}$} & GA & 16,363 & 16,363 - Align &\multirow{2}{1cm}{16,363$\ast$c\%}\\
	& EN-share & 19,388 + Align & 19,388\\
	\multirow{2}{1.3cm}{$\rm{GA-ZH}$} & GA & 16,363 & 16,363 - Align &\multirow{2}{1cm}{16,363$\ast$c\%}\\
	& ZH-share & 19,572 + Align & 19,572\\
	\multirow{2}{1.3cm}{$\rm{GA-JA}$} & GA & 16,363 & 16,363 - Align &\multirow{2}{1cm}{16,363$\ast$c\%}\\
	& JA-share & 19,814 + Align & 19,814\\
	\multirow{2}{1.3cm}{$\rm{GA-FR}$} & GA & 19,388 & 16,363 - Align &\multirow{2}{1cm}{16,363$\ast$c\%}\\
	& FR-share & 19,661 + Align & 19,661\\
	\bottomrule
\end{tabular}
\end{center}
\caption{Statistics of GA-DBP15K. c = [25\%,20\%,15\%,10\%].}
\label{tab:static_3}
\end{table}

\textbf{GA-DBP15K}: The GA-DBP15K dataset is derived from a subset of entities within GA16K, along with their associated triples, which are then concatenated with the DBP15K dataset, such as EN, resulting in a new dataset pair that shares a proportion of common entities. To achieve the goal, we first extract a certain proportion of triples from GA16K. We then reindex all the entities from the extricated GA16K and DBP15K datasets. Finally, we update the entity and relation indices in the triples, replacing them with the newly assigned indices.

\textbf{Baselines.} Since our work does not take advantage of any side information, we emphasize its comparison with the previous methods purely depending on graph structures. These works majorly incorporate two types:

\textit{Dangling-Entities-Unaware.} We include advanced entity alignment methods in recent years: GCN-Align \cite{wang2018cross}, RSNs \cite{guo2019learning}, MuGNN \cite{cao2019multi}, KECG \cite{li2019semi}. Methods with bootstrapping to generate semi-supervised structure data are also adopted: BootEA \cite{sun2018bootstrapping}, TransEdge \cite{sun2019transedge}, MRAEA \cite{mao2020mraea}, AliNet \cite{sun2020knowledge}, and Dual-AMN \cite{mao2021boosting}.

\textit{Dangling-Entities-Aware.} To the best of our knowledge, the method of \cite{sun2021knowing} is the most fairly comparable baseline which is based on MTransE \cite{chen2017multilingual} and AliNet \cite{sun2020knowledge}. Because MHP \cite{liu2022dangling} over-emphasized more use of labeled dangling data like high-order similarity information which is also based on the above two methods, while SoTead \cite{luo2021graph} and UED \cite{luo2022accurate} utilize additional side-information. SoTead \cite{luo2021graph} and UED \cite{luo2022accurate} can only execute the degraded version on DBP2.0 cause no side-information is available on that. We exclude them from baselines for our methods. \cite{sun2021knowing} introduces three techniques to address the dangling entity issue: nearest neighbor (NN) classification, marginal ranking (MR), and background ranking (BR).

\textbf{Metrics} are set for the dangling entity detection task and the entity alignment task separately. For the entity detection, we evaluate the detection performance by the standard precision, recall, and F1 score. To align the previous dangling detection baselines, we detect dangling entities as `positive' samples and align matchable entities for entity alignment.

For the entity alignment, the metrics slightly differ in the dangling-entities-unaware and dangling-entities-aware settings. We evaluate the baselines unaware of the dangling entities by following their assumptions and using their metric Hits@K (K$ = 1, 10, 50$, H@K for short) on the ranking list $S$. This setting is referred to as \textit{relaxed setting} when $S$ is composed of all ground-truth entities without the dangling ones:
\begin{equation*}
Hits@K= \frac{1}{|S|} \sum_{k=1}^{|S|}\mathbbm{1}(\textrm{rank}_i \le k).
\end{equation*}

In contrast, we refer to a \textit{consolidated setting} for baselines aware of dangling entities. In this setting, the ranking list $S$ also contains all dangling entities. We use H@K in the consolidated setting to evaluate the performance of baselines aware of but not removing dangling entities in the alignment. For baselines where dangling entities are detected and removed before alignment, the direct use of H@K to evaluate entity alignment may not be precise, since errors are introduced in the detection phase. Thus we follow the convention of  \cite{sun2021knowing} to apply a set of metrics for evaluating the accuracy of entity alignment in the consolidated setting. Each of them is derived and introduced as follows.

The standard precision and recall is given as 
\begin{equation*}
\textrm{precision}=\frac{TP}{TP+FP} ,~~ \textrm{recall}=\frac{TP}{TP+FN}
\end{equation*}
for dangling entity detection. We denote the dangling entities as 0 and matchable ones as 1. The subscript $t_{1y}$ suggests that an entity with ground truth $y$ is classified as matchable. Likewise, $t_{y1}$ represents a matchable entity that is classified as $y$ by the detection classifier. If a source entity is dangling but not identified by the detection module, its alignment result is always considered incorrect, i.e., $H@K_{t_{10}}=0$. Hence we have the precision for entity alignment as
\begin{equation}\label{eq:precAlign}
\begin{aligned}
H@1_{t_{1y}}&= \textrm{precision} \cdot H@1_{t_{11}}+(1-\textrm{precision} ) \cdot H@1_{t_{10}}\\
&=\textrm{precision}  \cdot H@1_{t_{11}}.
\end{aligned}
\end{equation}
Similarly, if a matchable entity is falsely excluded by the dangling detection module, this test case is also regarded as incorrect $H@K_{t_{01}}=0$ since the alignment model has no chance to search for alignment. Hence we have the recall for entity alignment as
\begin{equation}\label{eq:recAlign}
\begin{aligned}
H@1_{t_{y1}}&=\textrm{recall} \cdot H@1_{t_{11}}+(1-\textrm{recall}) \cdot H@1_{t_{01}}\\
&=\textrm{recall} \cdot H@1_{t_{11}}.
\end{aligned}
\end{equation}
For the methods that are aware of dangling entities, we use $H@1_{t_{1y}}$ and $H@1_{t_{y1}}$ to denote the precision and recall of the entity alignment task. Similarly, we define the F1 score of the entity alignment as the harmonic average of precision and recall:
\begin{equation}\label{eq:f1Align}
F1=\frac{2 \cdot H@1_{t_{1y}} \cdot H@1_{t_{y1}}}{H@1_{t_{1y}} + H@1_{t_{y1}}}.
\end{equation}
Later, $H@1_{t_{1y}}$ and $H@1_{t_{y1}}$ are referred to as Prec. and Rec. in reporting alignment performance.

\section{Additional Experiment}\label{app_8}
\textbf{RQ1:} How do current network alignment methods perform in unlabeled dangling cases? (see appendix~\ref{Non_Neg})

\textbf{RQ2:} Loss convergence on GA-DBP15K and DBP2.0. (see appendix~\ref{class_prior})

\textbf{RQ3:} How do we select the embedding dimensions? (see appendix~\ref{embedding_select}) 

\textbf{RQ4:} What is the actual efficiency of our approach? (see appendix~\ref{efficiency})

\textbf{RQ5:} Baseline comparison under different pre-aligned seeds? (see appendix~\ref{pre-aligned_seeds})

\textbf{RQ6:} Additional experiments involved LightEA as a strong baseline? (see appendix~\ref{lightea})

\subsection{The Non-Negligibility of dangling Problem (RQ1).}
\label{Non_Neg}
We investigated the performance degradation of various existing EA methods in the face of the dangling problem, which shows that this problem is worth considering.

\begin{table*}[htbp]
\resizebox{\textwidth}{!}{
\renewcommand\arraystretch{1.1}
\begin{tabular}{llllllllll}
	\toprule
	\multirow{2}{*}{Method} & \multicolumn{3}{c}{$\rm{DBP15K_{ZH-EN}}$} & \multicolumn{3}{c}{$\rm{DBP15K_{JA-EN}}$} & \multicolumn{3}{c}{$\rm{DBP15K_{FR-EN}}$} \\
	& H@1 & H@10 & H@50 & H@1 & H@10 & H@50 & H@1 & H@10 & H@50        \\
	\hline
	BootEA & 31.30\textcolor{blue}{$\downarrow 20.96$} &59.70\textcolor{blue}{$\downarrow 16.18$} &71.51\textcolor{blue}{$\downarrow 12.91$} &33.77\textcolor{blue}{$\downarrow 15.27$} &62.66\textcolor{blue}{$\downarrow 11.64$} &73.09\textcolor{blue}{$\downarrow 10.29$} &23.11\textcolor{blue}{$\downarrow 26.72$} &58.39\textcolor{blue}{$\downarrow 18.77$} &71.54\textcolor{blue}{$\downarrow 14.00$} \\
	TransEdge & 49.91\textcolor{blue}{$\downarrow 15.21$} &76.62\textcolor{blue}{$\downarrow 9.79$} &83.44\textcolor{blue}{$\downarrow 7.16$} &54.07\textcolor{blue}{$\downarrow 13.42$} &78.01\textcolor{blue}{$\downarrow 8.25$} &84.00\textcolor{blue}{$\downarrow 6.21$} &48.23\textcolor{blue}{$\downarrow 17.34$} &79.32\textcolor{blue}{$\downarrow 9.70$} &86.69\textcolor{blue}{$\downarrow 6.24$} \\
	MRAEA & 59.45\textcolor{blue}{$\downarrow 5.62$} &83.04\textcolor{blue}{$\downarrow 2.53$} &88.68\textcolor{blue}{$\downarrow 1.56$} &61.60\textcolor{blue}{$\downarrow 4.45$} &83.48\textcolor{blue}{$\downarrow 2.21$} &88.65\textcolor{blue}{$\downarrow 1.50$} &61.55\textcolor{blue}{$\downarrow 6.62$} &85.85\textcolor{blue}{$\downarrow 2.61$} &90.79\textcolor{blue}{$\downarrow 1.69$} \\
	GCN-Align & 31.99\textcolor{blue}{$\downarrow 10.70$} &62.21\textcolor{blue}{$\downarrow 6.45$} &71.93\textcolor{blue}{$\downarrow 4.31$} &32.08\textcolor{blue}{$\downarrow 10.08$} &61.04\textcolor{blue}{$\downarrow 5.86$} &70.34\textcolor{blue}{$\downarrow 3.52$} &30.71\textcolor{blue}{$\downarrow 10.50$} &61.64\textcolor{blue}{$\downarrow 7.07$} &72.45\textcolor{blue}{$\downarrow 5.55$} \\
	RSNs & 43.00\textcolor{blue}{$\downarrow 8.50$} &62.90\textcolor{blue}{$\downarrow 8.00$} &69.70\textcolor{blue}{$\downarrow 7.00$} &20.60\textcolor{blue}{$\downarrow 31.60$} &44.60\textcolor{blue}{$\downarrow 26.60$} &53.20\textcolor{blue}{$\downarrow 23.60$} &36.30\textcolor{blue}{$\downarrow 15.30$} &63.30\textcolor{blue}{$\downarrow 10.10$} &71.70\textcolor{blue}{$\downarrow 7.80$} \\
	MuGNN & 34.66\textcolor{blue}{$\downarrow 14.75$} &68.48\textcolor{blue}{$\downarrow 9.32$} &80.53\textcolor{blue}{$\downarrow 5.69$} &32.93\textcolor{blue}{$\downarrow 14.68$} &66.68\textcolor{blue}{$\downarrow 8.82$} &78.63\textcolor{blue}{$\downarrow 5.67$} &34.93\textcolor{blue}{$\downarrow 14.02$} &68.88\textcolor{blue}{$\downarrow 9.69$} &81.67\textcolor{blue}{$\downarrow 5.32$} \\
	KECG & 35.92\textcolor{blue}{$\downarrow 12.87$} &65.70\textcolor{blue}{$\downarrow 10.35$} &76.44\textcolor{blue}{$\downarrow 8.06$} &32.31\textcolor{blue}{$\downarrow 15.48$} &63.19\textcolor{blue}{$\downarrow 11.96$} &74.42\textcolor{blue}{$\downarrow 9.29$} &32.84\textcolor{blue}{$\downarrow 15.47$} &64.78\textcolor{blue}{$\downarrow 11.98$} &76.70\textcolor{blue}{$\downarrow 8.35$} \\
	AliNet & 53.84\textcolor{blue}{$\downarrow 0.66$} &73.73\textcolor{blue}{$\downarrow 3.16$} &80.30\textcolor{blue}{$\downarrow 1.59$} &52.69\textcolor{blue}{$\downarrow 1.30$} &74.01\textcolor{blue}{$\downarrow 2.60$} &80.91\textcolor{blue}{$\downarrow 1.90$} &54.01\textcolor{blue}{$\downarrow 0.58$} &76.19\textcolor{blue}{$\downarrow 2.74$} &83.25\textcolor{blue}{$\downarrow 1.40$}\\     
	Dual-AMN & 60.72\textcolor{blue}{$\downarrow 12.20$} &83.93\textcolor{blue}{$\downarrow 5.22$} &89.45\textcolor{blue}{$\downarrow 3.54$} &62.29\textcolor{blue}{$\downarrow 10.62$} &83.38\textcolor{blue}{$\downarrow 5.35$} &88.80\textcolor{blue}{$\downarrow 3.21$} &65.33\textcolor{blue}{$\downarrow 10.48$} &87.76\textcolor{blue}{$\downarrow 4.17$} &92.47\textcolor{blue}{$\downarrow 2.24$} \\
	\bottomrule
\end{tabular}
}
\caption{Network alignment performance on DBP15K in the consolidated setting. The blue numbers suggest the drop from the relaxed setting (as with their original implementation).}
\label{tab:cur_dangling_ent_alignment}
\end{table*}

We reproduce the baselines unaware of dangling entities on DBP15K in the relaxed setting. On the same dataset, we rerun their methods but in a consolidated setting that takes the dangling entities into account. Even though DBP15K only comprises a small percentage of dangling entities, the drop in the consolidated setting is significant, as shown in Tab.~\ref{tab:cur_dangling_ent_alignment}.

The reason behind such a performance drop is mainly because most previous works remove dangling entities from the ground truth in measuring their alignment performance. In particular, Dual-AMN takes advantage of the bootstrapping module by incorporating labeled pairs in training. In the relaxed setting, such labeled pairs are ground-truth aligned pairs, but in the consolidated setting, the dangling entities could bring in erroneous alignment which contaminates the alignment of other pairs.

\subsection{Class Prior Estimation Supplementary Experiment (RQ2).}\label{class_prior}
We hope further to verify the estimation and convergence results of iPULE of loss convergence. We list the corresponding loss convergence results in Fig.~\ref{fig:visual_EM_PU_appendix}. The losses under different pre-aligned proportions ($0.25, 0.2, 0.15, 0.1$) on the GA-DBP15K constitute a group of statistical data, and the corresponding loss mean and standard deviation of this set of statistical data are displayed. 

On the other hand, the loss difference is a direct indication of convergence in iPULE's implementation. Thus, we plot the histogram figure of the DBP2.0 (w/ -minus \& -plus) of the corresponding loss difference for more comprehensive. With the progress of the algorithm, and the statistical number of the difference of the smaller loss function occupied the maximum. This shows the convergence of iPULE in this data set from another aspect.

It is worth noticing that, there are performance fluctuations during the constitution of an ideal embedding space during the cold start stage. The figure plotted covers only the cold start subsequent procedure.

\begin{figure*}[htbp]
\begin{minipage}{0.53\linewidth}
\centering
\includegraphics[width=1.02\linewidth]{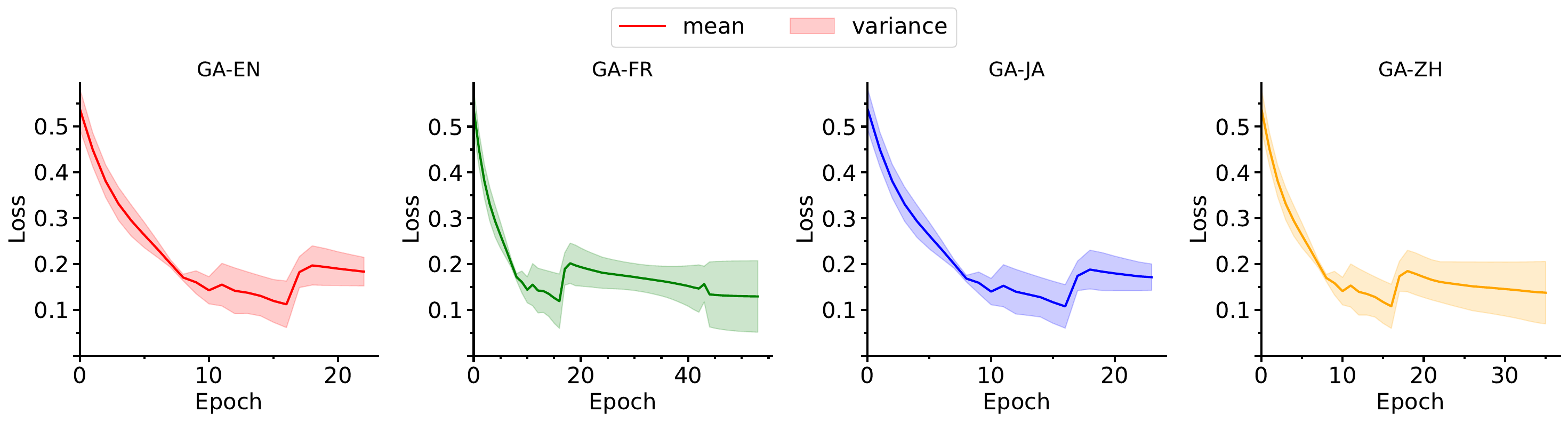}
\end{minipage}
\hfill
\begin{minipage}{0.46\linewidth}
\centering
\includegraphics[width=1\linewidth]{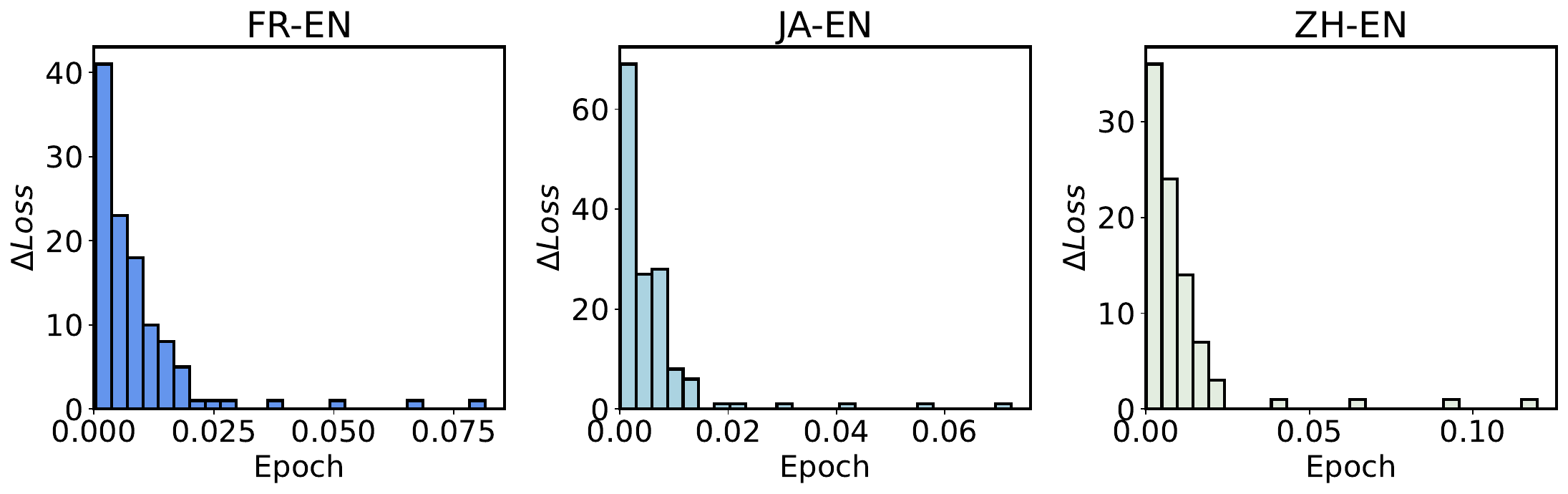}
\end{minipage}
\caption{Visualization of loss convergence on DBP2.0 and GA-DBP15K.}
\label{fig:visual_EM_PU_appendix}
\end{figure*}

\subsection{Embedding Dimension Selection (RQ3).}\label{embedding_select}
Although a higher embedding dimension may encode richer information, an overly high dimension leads to performance decline. We select the GNN dimension according to the principle of \cite{luo2021graph}. Let the dimension of embedding be $d$ and the number of entities is $N$. According to the feature entropy in \cite{luo2021graph}, it holds that $d>8.33\log N$ by the Johnson-Lindenstrauss lemma \cite{larsen2017optimality} that the vector dimension is at $\mathcal{O} (\log N)$ order. In most of our settings, $N$ is approximately $10^5$, and thus $d$ is set to 128.

\begin{table*}[htbp]
\centering
\renewcommand\arraystretch{0.9}
\resizebox{.99\textwidth}{!}{\small
{
	\small
	\setlength{\tabcolsep}{3pt}
	\begin{tabular}{llcccccccccccccccccc}
		\toprule
		\multicolumn{2}{c}{\multirow{2}{*}{Dimension}} &
		\multicolumn{3}{c}{ZH-EN} & \multicolumn{3}{c}{EN-ZH} & \multicolumn{3}{c}{JA-EN} & \multicolumn{3}{c}{EN-JA} &  \multicolumn{3}{c}{FR-EN} & \multicolumn{3}{c}{EN-FR}\\
		\cmidrule(lr){3-5} \cmidrule(lr){6-8} \cmidrule(lr){9-11} \cmidrule(lr){12-14} \cmidrule(lr){15-17} \cmidrule(lr){18-20}
		&& Prec. & Rec. & F1 & Prec. & Rec. & F1 & Prec. & Rec. & F1 & Prec. & Rec. & F1 & Prec. & Rec. & F1 & Prec. & Rec. & F1 \\ 
		\midrule
		\parbox[t]{2mm}{\multirow{3}{*}{\rotatebox[origin=c]{90}{{\small Ours}}}} 
		& 64  & .278 & .446 & .342 & .224 & .501 & .310 & .325 & .410 & .362 & .239 & .470 & .317 & .164 & .224 & .189 & .135 & .257 & .178 \\ 
		& 96  & \textbf{.281} & \textbf{.451} & \textbf{.346} & \textbf{.226} & \textbf{.505} & \textbf{.312} & .329 & .415 & .367 & .241 & .474 & .320 & .172 & .235 & .199 & .143 & .271 & .187 \\
		& 128 & .280 & .448 & .344 & .225 & .502 & .311 & \textbf{.330} & \textbf{.416} & \textbf{.368} & \textbf{.243} & \textbf{.477} & \textbf{.322} & \textbf{.177} & \textbf{.242} & \textbf{.204} & \textbf{.151} & \textbf{.287} & \textbf{.198} \\
		\bottomrule
\end{tabular}}}
\caption{The entity alignment performance over different embedding dimensions on DBP2.0.}
\label{tab:embedding selection}
\end{table*}

As shown in Tab.~\ref{tab:embedding selection}, due to the varying number of entities in datasets, the embedding dimension at the optimal performance varies. For example, the top performance is achieved on ZH-EN and EN-ZH when the embedding dimension is 96 but is obtained on JA-EN, EN-JA, FR-EN, and EN-FR with an embedding dimension of 128. As a compromise, we simplify the embedding representation of edges to enable FR-EN and EN-FR to run with an embedding dimension of 128 with limited memory (For more details, please check our open source code). A higher alignment performance can be achieved if no compromise is made. As we observe, the optimal performance is typically achieved at the theoretically chosen $d$. This also indicates our approach has a memory cost at the order of $\mathcal{O} (\log N)$.

\subsection{Efficiency (RQ4)}\label{efficiency}
The previous works concerning the dangling problem have not analyzed its efficiency in their experiments. Thus we only report the efficiency of our methods without baseline comparison. We evaluate the efficiency of our work on Tab.~\ref{tab:Efficiency} including alignment search time as `Inference Time', KEESA average training time as `Average Training Time', and GPU memory cost on three different datasets of DBP2.0. Data obtained from these three datasets with the top three node numbers is a robust indicator of the efficiency of our method. 

We gathered the mean value of 5 inference time costs for each dataset with the corresponding CPU and GPU memory consumption. Meanwhile, the average training time for each period from early to late is calculated. We enumerate the average training time of epochs 1-20, 21-25, 26-30, 31-35, 36-40, and 41-45.

Cause GPU is employed for not only model training but also inference, as shown on Tab.~\ref{tab:Efficiency}, the inference speed is still very impressive. Specifically, we split the large similarity matrix into multiple independent row blocks to perform the nearest searches within each block, which are well suited for GPU parallel processing. 

It's noteworthy that the average training time correspondingly increases as training progresses from early to late stages. More quasi-supervised information incorporated by us accounts for that. To be specific, as the training deepens, we repeatedly conduct preliminary alignment tests while we gather more and more entity pairs mutually closest under a given metric. The entity pairs serve as the pre-aligned anchor nodes, i.e., the quasi-supervisory information mentioned above.

Besides, we list the CPU and GPU memory consumption required for our work. Memory consumption is influenced by various factors such as complex allocation algorithms, model parameter scales, and hyperparameters. In this problem, we put more attention on triples which characterize one KG, revealing an approximate proportionality between the number of triples and memory consumption.

\begin{table*}[!h]
\centering
\renewcommand\arraystretch{1.0}
\resizebox{.99\linewidth}{!}
{
\small
\setlength{\tabcolsep}{3pt}
\begin{tabular}{llcc cccccc ccccccccccccccccccccc}
	\toprule
	\multicolumn{2}{c}{\multirow{2}{*}{Datasets}} &
	\multicolumn{1}{c}{\multirow{2}{*}{Triples}} &
	\multicolumn{1}{c}{\multirow{2}{*}{Inference Time}} &
	\multicolumn{6}{c}{Average Training Time (from 1 to 45 training epochs)} &
	\multicolumn{1}{c}{\multirow{2}{*}{CPU Memory}} &
	\multicolumn{1}{c}{\multirow{2}{*}{GPU Memory}} &\\
	\cmidrule(lr){5-10}
	& & & & 1-20 & 21-25 & 26-30 & 31-35 & 36-40 & 41-45 \\ 
	\midrule
	\multicolumn{2}{c}{$\rm{DBP2.0_{ZH-EN}}$} & 872,935 & 48.78s & 11.21s/it & 21.16s/it & 25.67s/it & 28.17s/it & 29.21s/it & 30.14s/it  & \multicolumn{1}{c}{10.8GB}  & \multicolumn{1}{c}{32.5GB} \\
	\multicolumn{2}{c}{$\rm{DBP2.0_{JA-EN}}$} & 1,015,545 & 120.76s & 28.14s/it & 53.99s/it & 63.43s/it & 68.27s/it & 70.61s/it & 72.80s/it  & \multicolumn{1}{c}{11.9GB} & \multicolumn{1}{c}{32.6GB} \\
	\multicolumn{2}{c}{$\rm{DBP2.0_{FR-EN}}$} & 2,089,909 & 382.48s & 90.18s/it & 158.18s/it & 190.65s/it & - & - & -  & \multicolumn{1}{c}{27.7GB} & \multicolumn{1}{c}{60.2GB} \\
	\bottomrule
\end{tabular}}
\caption{Efficiency performance of our work on DBP2.0. The measurement of average training time is `s/it', which indicates seconds per iteration. One iteration here represents one training epoch. `-' indicates the absence of data due to training termination.}
\label{tab:Efficiency}
\end{table*}

\subsection{Baseline Comparison Under Different Ratios of Pre-aligned Seeds (RQ5).}\label{pre-aligned_seeds}
Comparing the proposed method with strong baseline models under different ratios of pre-aligned seeds would better demonstrate Lambda's superiority. The experimental baseline includes MtransE w/ BR the SOTA method in previous works, which is also the only open-source method. The results are shown in the Table.~\ref{tab:pre_aligned_seeds}.

\begin{table*}[hbtp]
\centering
\renewcommand\arraystretch{0.8}
\resizebox{.99\textwidth}{!}{\small
{
	\small
	\setlength{\tabcolsep}{3pt}
	\begin{tabular}{llcccccccccccccccccc}
		\toprule
		\multicolumn{2}{c}{\multirow{2}{*}{Methods}} &
		\multicolumn{3}{c}{ZH-EN} & \multicolumn{3}{c}{EN-ZH} & \multicolumn{3}{c}{JA-EN} & \multicolumn{3}{c}{EN-JA} &  \multicolumn{3}{c}{FR-EN} & \multicolumn{3}{c}{EN-FR}\\
		\cmidrule(lr){3-5} \cmidrule(lr){6-8} \cmidrule(lr){9-11} \cmidrule(lr){12-14} \cmidrule(lr){15-17} \cmidrule(lr){18-20}
		& Ratios& Prec. & Rec. & F1 & Prec. & Rec. & F1 & Prec. & Rec. & F1 & Prec. & Rec. & F1 & Prec. & Rec. & F1 & Prec. & Rec. & F1 \\ 
		\midrule		
		\multicolumn{1}{c}{\multirow{3}{*}{MtransE w/ BR}} & 10\% & .161 & .141 & .148 & .105 & .127 & .114 & .102 & .102 & .128 & .102 & .128 & .113 & .133 & .085 & .102 & .096 & .072 & .083 \\
		& 20\% & \textbf{}{.264} & .267 & .265 & .186 & .251 & .213 & .179 & .180 & .251 & .180 & .251 & .210 & .215 & .151 & .179 & .167 & .138 & .150 \\
		& 30\% & \textbf{.312} & .362 & .335 & \textbf{.241} & .376 & .294 & .314 & .363 & .336 & \textbf{.251} & .358 & .295 & \textbf{.265} & .208 & .233 & \textbf{.231} & .213 & .222 \\
		\midrule
		\multicolumn{1}{c}{\multirow{3}{*}{Lambda}} & 10\% & \textbf{.236} & \textbf{.346} & \textbf{.280} & \textbf{.197} & \textbf{.385} & \textbf{.261} & \textbf{.262} & \textbf{.315} & \textbf{.286} & \textbf{.206} & \textbf{.360} & \textbf{.262} & \textbf{.179} & \textbf{.230} & \textbf{.201} & \textbf{.153} & \textbf{.260} & \textbf{.193} \\
		& 20\% & .262 & \textbf{.399} & \textbf{.316} & \textbf{.215} & \textbf{.446} & \textbf{.290} & \textbf{.300} & \textbf{.368} & \textbf{.330} & \textbf{.226} & \textbf{.417} & \textbf{.293} & \textbf{.217} & \textbf{.286} & \textbf{.247} & \textbf{.182} & \textbf{.324} & \textbf{.233} \\
		& 30\% & .279 & \textbf{.447} & \textbf{.344} & .219 & \textbf{.489} & \textbf{.303} & \textbf{.324} & \textbf{.409} & \textbf{.362} & .234 & \textbf{.460} & \textbf{.310} & .234 & \textbf{.320} & \textbf{.271} & .192 & \textbf{.363} & \textbf{.251} \\
		\bottomrule
\end{tabular}}}
\caption{Performance of Lambda and MtransE w/ BR under different ratios of pre-aligned of 10\%, 20\%, and 30\%. \textbf{Bold} indicates optimal performance.}
\label{tab:pre_aligned_seeds}
\end{table*}

\subsection{LightEA as Strong Baseline for Comparison (RQ6).}\label{lightea}
LightEA \cite{mao2022lightea} is recommended as a strong baseline for Lambda. We fixed LightEA's code to include dangling entities into the alignment candidates and evaluated its performance on DBP2.0. Hits@1 and Hits@10 are evaluated in a similar way to the dangling-unaware methods in our paper, as listed below. In comparison, Lambda still outperforms LightEA.

\begin{table*}[hbtp]
\centering
\renewcommand\arraystretch{0.5}
\resizebox{.55\textwidth}{!}{\small
{
	\small
	\setlength{\tabcolsep}{3pt}
	\begin{tabular}{llcccccccccccccccccc}
		\toprule
		\multicolumn{2}{c}{\multirow{2}{*}{Methods}} &
		\multicolumn{2}{c}{ZH-EN} & \multicolumn{2}{c}{JA-EN} & \multicolumn{2}{c}{FR-EN}\\
		\cmidrule(lr){3-4} \cmidrule(lr){5-6} \cmidrule(lr){7-8}
		&& H@1 & H@10 & H@1 & H@10 & H@1 & H@10 \\ 
		\midrule
		& LightEA &60.5\%	&82.9\%	&61.4\%	&\textbf{84.1\%}	&-	&-\\ 
		& Lambda  &\textbf{62.6\%}	 &\textbf{84.7\%}	 &\textbf{62.1\%}	 &84.0\%	 &\textbf{44.1\%}	 &\textbf{69.3\%} \\                
		\bottomrule
\end{tabular}}}
\caption{Comparison of Lambda and LightEA under relaxed setting. `-' indicates the absence of data due to out of time.}
\end{table*}

\section{Discussion}\label{sec:discussion}
\subsection{Alignment Direction}
As we found, the alignment problem with dangling cases has a deeper issue concerning the classification of imbalanced datasets. It originated from the observation that the alignment performance from the source to the target is different from the other direction. The work of \cite{sun2021knowing} has observed that on DBP2.0, choosing the alignment direction from a less populated KG (e.g., ZH, JA, FR) to a more populated KG (e.g., EN) enjoys a higher alignment accuracy but the other way around would lead to a noticeable performance drop. Meanwhile, the dangling entity detection on EN-XX has a higher F1 score than XX-EN, as shown in Tab.~\ref{tab: misclassification in DBP}.

\begin{table}[htbp]
\centering
\renewcommand\arraystretch{0.7}
\setlength{\tabcolsep}{5pt}
\small
\begin{tabular}{llccccccccccccccccccccccccccc}
\toprule
\multicolumn{2}{c}{\multirow{3}{*}{Datasets}} &
\multicolumn{6}{c}{Dangling Detection} & 
\multicolumn{3}{c}{Entity Alignment} &\\
\cmidrule(lr){3-8} \cmidrule(lr){9-12}
\multicolumn{2}{c}{} & \multicolumn{3}{c}{Our Work} & \multicolumn{3}{c}{Trivial} & \multicolumn{3}{c}{Our Work}\\
\cmidrule(lr){3-5} \cmidrule(lr){6-8} \cmidrule(lr){9-12}
&&  Prec. & Rec. & F1 & Prec. & Rec. & F1 & Prec. & Rec. & F1\\ 
\midrule
\multicolumn{2}{c}{ZH-EN} & .763 & \textbf{.925} & .836 & .583 & \textbf{1} & .736 & \textbf{.279} & .447 & \textbf{.344}\\
\multicolumn{2}{c}{EN-ZH} & \textbf{.844} & .909	& \textbf{.875} & \textbf{.609} & \textbf{1} & \textbf{.756} & .219 & \textbf{.489} & .303\\
\midrule
\multicolumn{2}{c}{JA-EN} & .807 & \textbf{.836} & .821 & .580 & \textbf{1} & .734 & \textbf{.324} & \textbf{.409} & \textbf{.362}\\
\multicolumn{2}{c}{EN-JA} & \textbf{.880} & .809	& \textbf{.843} & \textbf{.605} & \textbf{1} & \textbf{.753} & .234 & .320 & .271\\
\midrule
\multicolumn{2}{c}{FR-EN} & .615 & \textbf{.772} & .685 & .439 & \textbf{1} & .610 & \textbf{.234} & .320 & \textbf{.271}\\
\multicolumn{2}{c}{EN-FR} & \textbf{.732} & .749	& \textbf{.740} & \textbf{.554} & \textbf{1} & \textbf{.715} & .192 & \textbf{.363} & .251\\
\bottomrule
\end{tabular}
\caption{Dangling entities detection by our classifier v.s. a trivial one on DBP2.0.}
\label{tab: misclassification in DBP}
\end{table}

By analysis, we think it may be attributed to an improper indication of the dangling entity detection power on imbalanced datasets. This error in removing the predicted dangling entity would accumulate hurting the alignment task. To verify the point, we introduce a trivial classifier that makes a simple choice to classify all entities as dangling (positive) ones, and the detection results are reported in Tab.~\ref{tab: misclassification in DBP}. As all unlabeled entities are trivially classified as dangling ones, the detection metrics of the trivial classifier are all falsely high. The more populated source KG usually has more dangling entities (positives) and thus yields a higher precision in detection. Meanwhile, since the detection classifier actually is not working, more dangling entities participate in the alignment phase, resulting in poor alignment performance. This has explained why EN-XX has a higher dangling detection performance but a lower alignment accuracy compared to the other direction.  

The root of this issue is that matchable and dangling entities comprise imbalanced categories in the classification task, but the corresponding metric is inappropriate. Hence boosting the detection performance does not necessarily improve the alignment performance. We believe more practical indicators of imbalanced datasets should be introduced to the alignment problem.








\subsection{The Similarity between Dual-AMN and Lambda}
The differences between the proposed GNN and Dual-AMN include:

\textbf{Aggregation}: 

1.  The adaptive dangling indicator $r_{e_j}$ is included in Lambda for eliminating dangling pollution.

2.  The indicator $r_{e_j}$ is concatenated as a part of the entity feature.

\textbf{Attention}: 

1.  The attention is scaled by $r_{e_j}$ to filter dangling information.

2.  Relation $r_k$'s embedding $\bf{h}_{r_k}$ is linked to the adaptive dangling indicator of the associated entity $r_{e_j}$, and thus the attention in Eq.~(\ref{attention}) models the relationship between the relation and the entity.


\newpage
\section*{NeurIPS Paper Checklist}

\begin{enumerate}
	
	\item {\bf Claims}
	\item[] Question: Do the main claims made in the abstract and introduction accurately reflect the paper's contributions and scope?
	\item[] Answer: \answerYes{} 
	\item[] Justification: see \textbf{Abstract} and \textbf{Introduction.~\ref{Introduction}}.
	
	\item {\bf Limitations}
	\item[] Question: Does the paper discuss the limitations of the work performed by the authors?
	\item[] Answer: \answerYes{} 
	\item[] Justification: We explained the trade-off of our method in \textbf{How does our method work?~\ref{sec:visual}} for a slightly inferior precision reported in Tab.~\ref{tab:our_dangling_ent_alignment} and Tab.~\ref{tab:alignment_ent_alignment}.
	
	\item {\bf Theory Assumptions and Proofs}
	\item[] Question: For each theoretical result, does the paper provide the full set of assumptions and a complete (and correct) proof?
	\item[] Answer: \answerYes{} 
	\item[] Justification: According to the order of appearance, we sort out and give the specific proof in the appendix. 
	\begin{itemize}
		\item The problem setting about PU learning as sec.~\ref{sec:PU_setting}
		\item Proof for Lemma~\ref{theorem_converge}. 
		\item Proof for Theorem~\ref{theorem_unbiased}.
		\item Proof for Theorem~\ref{deviation_bound}.
		\item Proof for Theorem~\ref{theorem_Q_Rpu}. 
		\item Proof for Lemma~\ref{spec_hard}. 
	\end{itemize}
	
	
	\item {\bf Experimental Result Reproducibility}
	\item[] Question: Does the paper fully disclose all the information needed to reproduce the main experimental results of the paper to the extent that it affects the main claims and/or conclusions of the paper (regardless of whether the code and data are provided or not)?
	\item[] Answer: \answerYes{} 
	\item[] Justification: We introduce the method proposed in this paper in detail in two sections, \textbf{Selective Aggregation with Spectral Contrastive Learning}~\ref{sec:KEESA} and \textbf{Iterative Positive-Unlabeled Learning for Dangling Detection}~\ref{sec:iPULE}, and use the Alg.~\ref{EM-PU} to describe the latter in pseudocode. Meanwhile, we gave implementation details at the beginning of the Experiment ~\ref{exp_begin}. 

	\item {\bf Open access to data and code}
	\item[] Question: Does the paper provide open access to the data and code, with sufficient instructions to faithfully reproduce the main experimental results, as described in supplemental material?
	\item[] Answer: \answerYes{} 
	\item[] Justification: We provide the code and data in supplemental material which is described in a documented readme.md file.

	\item {\bf Experimental Setting/Details}
	\item[] Question: Does the paper specify all the training and test details (e.g., data splits, hyperparameters, how they were chosen, type of optimizer, etc.) necessary to understand the results?
	\item[] Answer: \answerYes{} 
	\item[] Justification: We gave the main implementation details at the beginning of the Experiment ~\ref{exp_begin}. \textit{Statistics of the experimental dataset and baselines} in appendix~\ref{app_7} and \textit{additional experiment} in appendix~\ref{app_8} also cover that including dataset construction details and hyperparameter selection criteria. 
	
	\item {\bf Experiment Statistical Significance}
	\item[] Question: Does the paper report error bars suitably and correctly defined or other appropriate information about the statistical significance of the experiments?
	\item[] Answer: \answerYes{} 
	\item[] Justification: We provide the corresponding mean and standard deviation curves in Fig.~\ref{fig:visual_EM_PU_appendix} by calculating the loss function of different alignment ratios $0.25, 0.2, 0.15, 0.1$, and the corresponding mean and standard deviation are drawn. Other experimental data have also been measured many times to take the mean value.
	
	\item {\bf Experiments Compute Resources}
	\item[] Question: For each experiment, does the paper provide sufficient information on the computer resources (type of compute workers, memory, time of execution) needed to reproduce the experiments?
	\item[] Answer: \answerYes{} 
	\item[] Justification: We gave GPU and CPU resources needed for the experiment in \textbf{Implementation Detail} part at the beginning of the Experiment ~\ref{exp_begin}. Additionally, time of execution such as training \& inference time is provided in Efficiency.~\ref{efficiency}.
	
	\item {\bf Code Of Ethics}
	\item[] Question: Does the research conducted in the paper conform, in every respect, with the NeurIPS Code of Ethics \url{https://neurips.cc/public/EthicsGuidelines}?
	\item[] Answer: \answerYes{} 
	\item[] Justification: The dataset construction and usage do not contain any information that endangers personal privacy, and it is licensed.

	\item {\bf Broader Impacts}
	\item[] Question: Does the paper discuss both potential positive societal impacts and negative societal impacts of the work performed?
	\item[] Answer: \answerNA{} 
	\item[] Justification: There is no societal impact of the work performed.
	
	\item {\bf Safeguards}
	\item[] Question: Does the paper describe safeguards that have been put in place for responsible release of data or models that have a high risk for misuse (e.g., pretrained language models, image generators, or scraped datasets)?
	\item[] Answer: \answerNA{} 
	\item[] Justification: The paper poses no such risks.
	
	\item {\bf Licenses for existing assets}
	\item[] Question: Are the creators or original owners of assets (e.g., code, data, models), used in the paper, properly credited and are the license and terms of use explicitly mentioned and properly respected?
	\item[] Answer: \answerYes{} 
	\item[] Justification: In this paper, we give all the sufficient reference materials. We provide the code and data in the supplemental material and describe them in a documented readme.md file, where more required information is clarified. 
	
	\item {\bf New Assets}
	\item[] Question: Are new assets introduced in the paper well documented and is the documentation provided alongside the assets?
	\item[] Answer: \answerYes{} 
	\item[] Justification: We introduce the dataset GA16K, GA-DBP15K and DBP2.0-minus \& -plus in detail in the appendix~\ref{app_7}. We provide the code and data in the supplemental material and describe them in a documented readme.md file, where more required information is clarified.
	
	\item {\bf Crowdsourcing and Research with Human Subjects}
	\item[] Question: For crowdsourcing experiments and research with human subjects, does the paper include the full text of instructions given to participants and screenshots, if applicable, as well as details about compensation (if any)? 
	\item[] Answer: \answerNA{} 
	\item[] Justification: The paper does not involve crowdsourcing nor research with human subjects.
	
	\item {\bf Institutional Review Board (IRB) Approvals or Equivalent for Research with Human Subjects}
	\item[] Question: Does the paper describe potential risks incurred by study participants, whether such risks were disclosed to the subjects, and whether Institutional Review Board (IRB) approvals (or an equivalent approval/review based on the requirements of your country or institution) were obtained?
	\item[] Answer: \answerNA{} 
	\item[] Justification: The paper does not involve crowdsourcing nor research with human subjects.
\end{enumerate}

\end{document}